%% file: AAAI_23.tex
\documentclass[letterpaper]{article} 
\usepackage{aaai23}  
\usepackage{times}  
\usepackage{helvet}  
\usepackage{courier}  
\usepackage[hyphens]{url}  
\usepackage{graphicx} 
\usepackage{natbib}  
\usepackage{caption} 
\usepackage{url}            
\usepackage{booktabs}       
\usepackage{nicefrac}       
\usepackage{xcolor}         
\usepackage{placeins}
\usepackage[export]{adjustbox}

 \usepackage[ruled]{algorithm2e} 

\usepackage{caption}
\usepackage{enumitem}

\def\PTD{\textsc{P-TD}\xspace}
\def\PTDD{\textsc{P-TD}$^D$\xspace}
\def\PTDAK{\textsc{P-TD}$^{AK}$\xspace}

\def\DTD{\textsc{D-TD}\xspace}

\def\IPTD{\textsc{iP-TD}\xspace}

\def\MAS{\textsc{MAS}\xspace}
\def\DSTD{\textsc{DS}\xspace}
\def\CRH{\textsc{CRH}\xspace}
\def\BPTD{\textsc{IBP}\xspace}

\def\EVTD{\textsc{EVD}\xspace}
\def\TT{\textsc{TOP2}\xspace}
\def\TTEXP{\textsc{EXP}\xspace}
\def\CATD{\textsc{CATD}\xspace}
\def\GTM{\textsc{GTM}\xspace}
\def\KDEm{\textsc{KDE}\xspace}
\def\UA{\textsc{UA}\xspace}
\def\OA{\textsc{OA}\xspace}

\newcommand{\inpaper}[1]{}
\def\etal{\textit{et al.}\xspace}
\usepackage{amsmath}
\usepackage{amssymb}
\usepackage{bm}
\usepackage{amsthm}
\usepackage{blindtext}
\usepackage{graphicx}
\graphicspath{ {images/} }
\usepackage{verbatim}
\usepackage{xcolor}
\usepackage{stmaryrd}
\usepackage{arydshln}
\usepackage{tikz}
\input{defs}

\usepackage{verbatim}
\def\cite{\citep}
\def\shortcite{\citeyearpar}
\newcount\isfull
\newcommand{\full}[2]{
\ifnum\isfull=1{#1}\else{#2}\fi}
\newcount\Comments  
\definecolor{darkgreen}{rgb}{0,0.6,0}
\newcommand{\kibitz}[2]{\ifnum\Comments=1{\color{#1}{#2}}\fi}
\newcommand{\rmr}[1]{\kibitz{blue}{[RESHEF:#1]}}

\newcommand{\lirong}[1]{\kibitz{orange}{[LX:#1]}}

\newcommand{\tsv}[1]{\kibitz{purple}{[TSVIEL:#1]}}
\newcommand{\cut}[1]{}
\def\({\left(}
\def\){\right)}

\def\newpar{\vspace{-0mm}\paragraph}

\newcommand{\newsec}[1]{\vspace{-0mm}\section{#1}\vspace{-0mm}}
\def\subsec{\subsection}
\newcommand{\hatv}[1]{{\ddot{#1}}}
\def\bm{\boldsymbol}
\newcommand{\hatvm}[1]{\hat{\bm{#1}}}

\newtheorem*{theorem*}{Theorem}

\def\calY{\mathcal{Y}}
\def\calZ{\mathcal{Z}}
\def\bias{\bm{b}}
\def\b{b}

\def\vpl{\vphantom{$a_{a_{a_{a_{a_{a_{a_a}}}}}}$}}
\def\vph{\vphantom{$a^{a^{a^{a^{a^a}}}}$}}
\newcommand{\aboveUA}[1]{{\color{gray}{{#1}}}}
\newcommand{\inrange}[1]{{\textbf{#1}}}
\newcommand{\best}[1]{{\underline{\textbf{#1}}}}
\def\irrel{$-$}

\title{Frustratingly Easy Truth Discovery}
\author {
    Reshef Meir,\textsuperscript{\rm 1} 
    Ofra Amir,\textsuperscript{\rm 1}
    Omer Ben-Porat,\textsuperscript{\rm 1} 
    Tsviel Ben-Shabat,\textsuperscript{\rm 1} 
    Gal Cohensius,\textsuperscript{\rm 1}
    Lirong Xia\textsuperscript{\rm 2}
}
\affiliations {
    \textsuperscript{\rm 1}  Technion---Israel Institute of Technology\\
    \{reshefm, oamir, omerbp\}@ie.technion.ac.il, \{tsviel,galcohensius\}@gmail.com\\
    \textsuperscript{\rm 2} RPI, 
     xial@cs.rpi.edu
}

\begin{document}
\maketitle

	\begin{abstract}
Truth discovery is a general name for a broad range of statistical methods aimed to extract the correct answers to questions, based on multiple answers coming from noisy sources. For example, workers in a crowdsourcing platform.
In this paper, we consider an extremely simple heuristic for estimating workers' competence using  average proximity to other workers. We prove that this  estimates well the actual competence level and enables separating high and low quality workers in a wide spectrum of domains and statistical models. Under Gaussian noise,  this simple estimate is the unique solution to the Maximum Likelihood Estimator with a constant regularization factor.  

Finally,  weighing workers according to their average proximity in a crowdsourcing setting, results in  substantial improvement over unweighted aggregation and other truth discovery algorithms in practice.
\smallskip

\begin{small}
	\begin{quote}
    ``All happy families are alike; each unhappy family is unhappy in its own way."
\flushright--- Leo Tolstoy, Anna Karenina
\end{quote}
\end{small}
	\end{abstract}

	
\vspace{-0mm}
\newsec{Introduction} \label{sec:intro} 
Consider a standard crowdsourcing task such as identifying which images contain a person or a car~\cite{deng2014scalable}, or identifying the location in which pictures were taken~\cite{NASA}. 
Such tasks are also used to construct large datasets that can later be used to train and test machine learning algorithms. Crowdsourcing workers are usually not experts, thus answers obtained this way often contain many mistakes \cite{vuurens2011much,wais2010towards}, and multiple answers are aggregated to improve accuracy. 

From a theory/statistics perspective, ``truth discovery" is a general name for a broad range of methods that aim to extract some underlying ground truth from noisy answers. While the mathematics of truth discovery dates back to the early days of statistics, at least to the \emph{Condorcet Jury Theorem}~\cite{condorcet1785essai}, the rise of crowdsourcing platforms suggests an exciting modern application of aggregating \emph{complex labels} from varied domains such as image processing and  natural language, to healthcare. For example, the Etch-a-Cell project uses volunteers to trace the boundary of tumors on Electron Microscopy images (\cite{spiers2021deep}, see Fig.~\ref{fig:etch_example}). 

Yet, the vast majority of the theoretical literature on truth discovery follows 
Condorcet by focusing on binary, multi-label or sometimes real-valued questions (see Related Work section), while specific applications with complex labels often rely on specialized algorithms. 

Many of these algorithms aim to identify first the most competent workers. While some of them employ highly sophisticated analysis, others are much more direct: for example, \citet{kobayashi2018frustratingly} suggests a `frustratingly easy' algorithm that ranks workers by their \emph{average cosine similarity} to others in a text summarization task; and \citet{kurvers2019detect} prove that the \emph{Hamming distance} of a worker from others is correlated with her competence in answering yes/no questions.  Of course, using average similarity or distance is not a new  idea, and is extensively employed outside the context of aggregation, for example in \emph{Games with a Purpose}~\cite{von2008designing,huang2013enhancing} to identify outliers, and in peer prediction to incentivize effort~\cite{witkowski2013dwelling}. 

In this paper we argue that average similarity is a powerful tool, with nothing special about Cosine or Hamming similarity in particular. 
 Our main observation can be written as follows:
\begin{theorem*}[Anna Karenina principle, informal]\label{thm:AK}
The expected average similarity of each worker to all others, is roughly linearly increasing in her competence.
\end{theorem*}
 Essentially,  the theorem says that as in Tolstoy's novel, ``good workers are all alike,''  whereas ``each bad worker is bad in her own way'' and thus not  similar to other workers.

  \subsec{Contribution and paper structure}

 

 \lirong{This sentence is a little weak. Can we say more explicitly how our work is better than previous work? For example, can we briefly mention the models and say that the Anna Karenina principles provide a systematic way of estimating fault, which often leads algs with good theoretical guarantee (e.g. consistency), and often outperform existing algs in practice based on large-scale exp?}\rmr{better?}

 After the preliminary definitions in  Section~\ref{sec:prelim}, we prove a formal version of the Anna Karenina principle and show how it can be used to identify poor workers in Section~\ref{sec:AK} without assuming specific label structure. 
 We show how additional assumptions lead to tighter corollaries of exactly or approximately linear relation between  pairwise similarity and competence. 
 To the best of our knowledge these are the first formal guarantees on  general-domain truth discovery. 
  
 
In Section~\ref{sec:AWG} we focus on the widely studied case of Gaussian noise.  We prove that  the average distance to other workers coincides with the Maximum Likelihood Estimator (MLE) for workers' (in)competence---the first guarantee of this type regarding average similarity or distance.

 In Section~\ref{sec:PTD} we explain how to leverage the Anna Karenina principle for aggregation  using a simple algorithm (\PTD). We  demonstrate  on real and synthetic data, that \PTD   substantially improves aggregation accuracy, competing well with advanced and domain-specific algorithms. 

Most proofs, as well as additional empirical results are \full{ relegated to  the appendices.}{ available in the full version of the paper on arXiv: \\ \url{https://arxiv.org/abs/1905.00629}.}

\subsection{Related work}\label{sec:related}The Condorcet Jury Theorem~\cite{condorcet1785essai} was perhaps the first formal treatment of truth discovery, and extensions   to  experts with heterogeneous competence levels were surveyed by \citet{grofman1983thirteen}. 
The idea of estimating workers' competence in order to improve aggregation is thus underlying many of the algorithms in the area (a recent survey is in \cite{li2016survey}). 
We should note that \emph{self-reporting} of accuracy often leads to poor results~\cite{gadiraju2017using,prelec2017solution}.

\newpar{Average similarity}
\rmr{the Science paper on binary labels, Gal's workshop paper w/o theoretical guarantees, practical use in crowdsourcing}
We have mentioned in the introduction the two applications of average similarity to truth discovery that we are aware of. Both of them assume a specific label structure and (somewhat surprisingly) both are quite recent: 
Kobayashi~\shortcite{kobayashi2018frustratingly} proved that cosine similarity  approximates a known kernel density estimator. \citet{kurvers2019detect} focused on \emph{binary questions with independent errors}, showing both theoretically and empirically that the expected average \emph{Hamming proximity} correlates with the true competence, albeit without comparing to any other algorithm. 

Our Anna Karenina theorem entails the Kurvers et al.  result  as a special case, and provides explicit performance guarantees for the heuristic suggested by Kobayashi.
\rmr{
 We show how their theoretical result can be derived from our Anna Karenina theorem as a special case.}

\newpar{Domain-specific algorithms}
Many truth-discovery algorithms have been proposed for specific label structures, mostly for categorical (multiple-choice) and real-valued labels. Often these algorithms entwine accuracy and ground truth  estimation, by iteratively aggregating labels to obtain an estimate of the ground truth, and using that in turn to estimate workers competence. This approach was pioneered by the EM-style Dawid-Skene estimator~\cite{dawid1979maximum}, with many follow-ups~\cite{karger2011iterative,gao2013minimax,aydin2014crowdsourcing,xiao2016truth,zhao2012probabilistic,
li2012truth}. 

Another class of algorithms  uses spectral methods to infer the competence and/or other latent variables from the covariance matrix of the workers~\cite{parisi2014ranking,zhang2016spectral}, or from their pairwise Hamming similarity~\cite{li2018incorporating}.  Note that covariance can also be thought of as a measure worker similarity in the context of binary labels. 
In rank aggregation, every voting rule can be considered as a truth-discovery algorithm~\cite{mao2013better,caragiannis2013noisy}.

Some of these works also provide formal convergence guarantees and/or bounds on the error that are subject to assumptions on the distribution of answers.

\newpar{General labels}
When there are complex labels that are not numbers or categories, but for example  contain text, graphics and/or hierarchical structure, there may not be a natural way to aggregate them but we would still want to evaluate workers' competence. 


Two recent papers suggest to use the pairwise distance (or similarity) matrix as a general domain-independent abstraction, then applying sophisticated algorithms on this matrix:
  The \emph{multidimensional annotation scaling} (MAS) model~\cite{braylan2020modeling} extends the Dawid-Skene model by calculating the labels and competence levels that would maximize the likelihood of the observed distance matrix, using the  Stan probabilistic programming language; 
Another approach is to find a `core' of good workers~\cite{kawase2019graph}, by looking for a dense subgraph of the similarity  matrix.

While we adopt the approach that \emph{pairwise similarity is the right domain-independent abstraction} for general labels, we argue that usually there is \underline{no need} for such complex algorithms: a `frustratingly easy' average is sufficient. 
\newsec{Preliminaries}\label{sec:prelim}

We consider a set $N$ of $n$ workers, 
each providing a report in some space $Z$. We denote elements of $Z$ (typically $m$-length vectors, see below) in \textbf{bold}.
 Thus, an instance of a truth discovery is a pair $\tup{S=(\bm s_i)_{i\in N},\bm z}$, where $\bm s_i\in Z$ is the report of worker $i$, and $\bm z\in Z$ is the \emph{ground truth}. $S$ is also called a \emph{dataset}.\footnote{It is ok if $\bm s_i$ is a partial vector, as long as there is enough intersection between pairs of workers.}    
 
 \newpar{Noise model}
 We do not make any assumptions regarding the ground truth $\bm z$. The \emph{type} $t_i$ of a worker determines her distribution of answers.
  A dataset is constructed in two steps: 
  \begin{enumerate}[label=(\arabic*)]
      \item  Sample a finite \emph{population} of  workers  i.i.d from a distribution $\calT$  (called a \emph{proto-population}) over a set of types $T$. For our running example, suppose that $\calT$ is  uniform  over $[50,200]$, $n=5$ and sampled types are $\vec t\!=\!(55, 80, 100,120,165)$, where lower types will provide better estimation in expectation (note that we use an arrow accent for $n$-length vectors).
      \item Workers each report their answers $S$, which depend on the ground truth $\bm z$, on their types, and on a random factor. $\bm z$ and $S$ for our example are shown in Table~\ref{tab:s_example} and Fig.~\ref{fig:s_example_main}. \full{See a more detailed example in  Appendix~\ref{apx:example}.}{}
  \end{enumerate}
  

Formally, a \emph{noise model} is a function $\calY:Z\times T\rightarrow \Delta(Z)$. That is,
the report of worker $i$ is  a random variable $\bm s_i$ sampled from $\calY(\bm z,t_i)$. 
 We note that $\calT$, $\calY$ and $\bm z$ together induce a distribution $\calY(\bm z,\calT)$ over answers (and thus over datasets), where $\bm s \sim \calY(\bm z,\calT)$ means we first sample a type $t\sim \calT$ and then a report $\bm s\sim \calY(\bm z,t)$.

The data in our example (Table~\ref{tab:s_example}) was sampled from the noise model $\calY$ that is a multivariate independent Normal distribution with mean $\bm z$ and variance $t_i$.  This is known as \emph{Additive White Gaussian noise} (AWG, see~\cite{diebold1998elements}). 





\full{
We can also think of more general models: e.g. where $t_i$ is a covariance matrix (capturing dependency among questions); where the type of a worker may include a constant bias $\bm\alpha_i$ (thus $\bm s_i=\bm z+\bm\alpha_i +\bm\eps_i$); or where different noise is added depending on the ground truth. 
}{}
\newpar{Workers' competence} Competent workers are close to the truth. More formally, 
given some ground truth $\bm z$ and a distance measure $d$, we define the \emph{fault} (or \emph{incompetence}) of a worker~$i$ as her expected distance from the ground truth, denoted $f_i(\bm z) := E_{\bm s_i\sim \calY(\bm z,t_i)}[d(\bm s_i,\bm z)]$.

\begin{figure}
\centering
 \includegraphics[width=6cm]{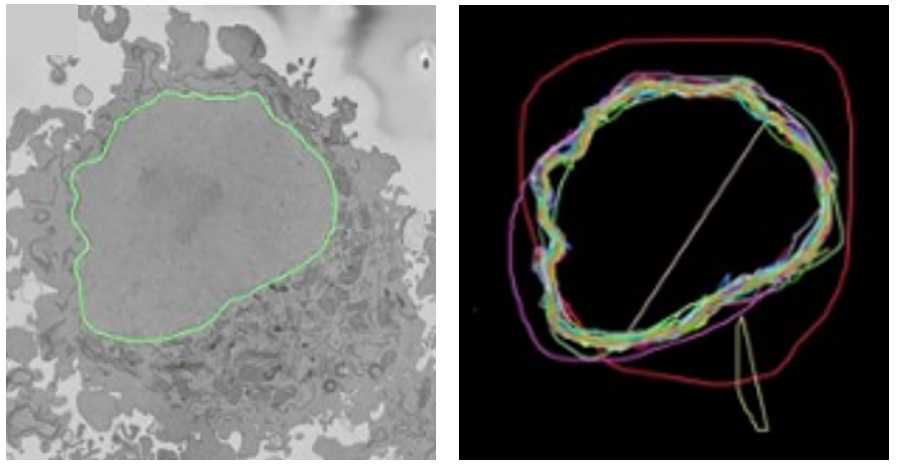}
  \caption{\label{fig:etch_example}The Etch-a-Cell project. Left: an Electron Microscopy (EM) image of a cell. The real boundary of the tumor is marked in green. Right: multiple annotations by volunteers. Images taken from~\cite{spiers2021deep}.}
\end{figure}

\begin{table}
 \begin{scriptsize}
 \begin{center}
  $ \begin{array}{c|c|cccc|c}
  i ~\backslash~ j&  & 1 & 2 & 3& 4& d(\bm s,\bm z)\\ \hline
  & t_i~\backslash~ z_j & 80 & 0 & 40 & -10 & \\
  \hline
1	&	55	&	81	&	6	&	41	&	-14	&	13.5	\\
2	&	80	&	89	&	-6	&	35	&	4	&	84.5	\\
3	&	100	&	105	&	-18	&	39	&	-5	&	243.7	\\
4	&	120	&	68	&	9	&	62	&	-10	&	177.2	\\
5	&	165	&	67	&	20	&	58	&	-20	&	248.2	\\
			\hline
UA(S)&		&	82	&	2.2	&	47	&	-9	&	14.7	\\
    \end{array}$
    \end{center}
    \end{scriptsize}\vspace{-0mm}
  
    \caption{\label{tab:s_example}An example of a dataset  sampled from the AWG model. The  bottom row is showing aggregated results using the unweighted mean. The rightmost column shows the \emph{error}, i.e. the distance of every row from the ground truth. 
    }
    \vspace{-0mm}
    \end{table}
    g
    \begin{figure}[t]
 \includegraphics[width=8.4cm]{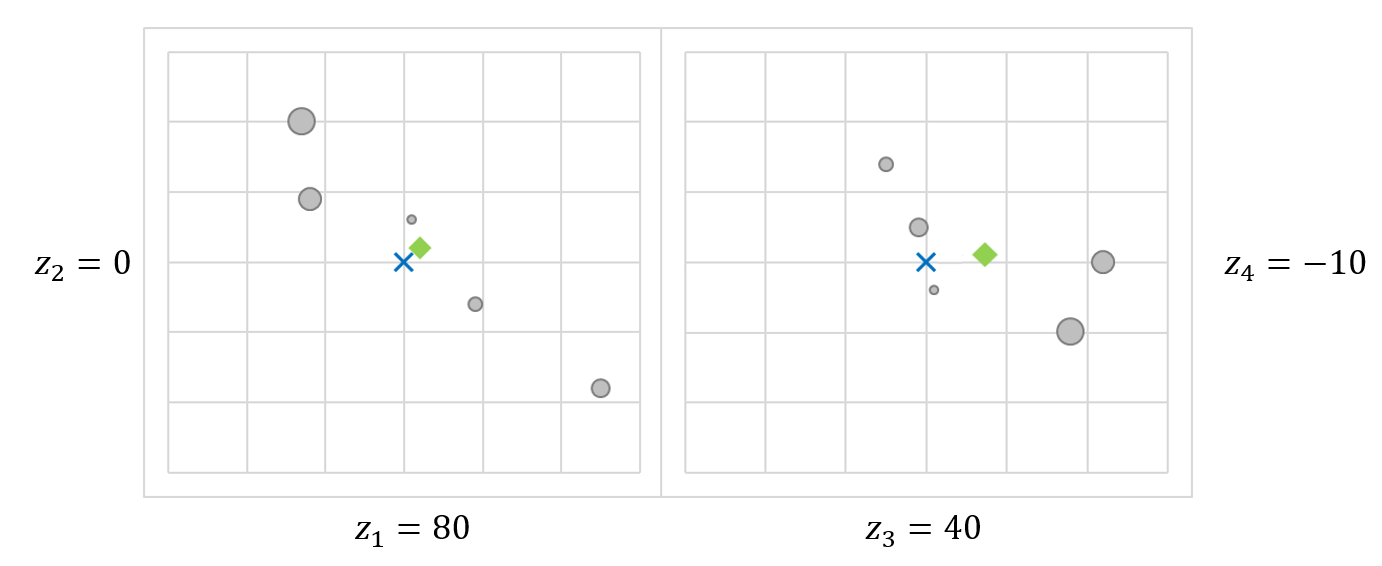}
  \caption{\label{fig:s_example_main} A graphical representation of Table~\ref{tab:s_example} (for obvious reasons we use two 2-dimensional plots instead of a 4-dimensional one). The blue X marks the ground truth. Workers' reports are marked by gray circles, whose size is proportional to $t_i$ (so smaller circles tend to be closer to the truth). The mean $\UA(S)$ is marked by a green diamond.}
    \vspace{-0mm}
    \end{figure}
 We denote by $\mu_\calT(\bm z):=E_{\bm s\sim \calY(\bm z,\calT)}[d(\bm s,\bm z)]$ the mean fault, omitting $\calT$ and/or $\bm z$ when clear from the context. 

Distance measures can often be derived from an inner product. Formally, consider an arbitrary symmetric inner product space  $(Z,\tup{\cdot,\cdot})$. 
This  induces a norm $\|\bm x\|^2:=\tup{\bm x,\bm x}$ and a distance measure  $d(\bm x,\bm y):=\|\bm x-\bm y\|^2$ (not necessarily a metric). A special case of interest is the {normalized Euclidean product} on $Z=\mathbb R^m$, defined as $\tup{\bm x,\bm y}_E:=\frac1m\sum_{j\leq m}x_jy_j$; and the corresponding 
\emph{normalized squared Euclidean distance} (NSED), a natural  way to capture the dissimilarity of two items~\cite{carter1989partitioning}.  
Note that the fault of a worker in the AWG model under NSED is  her \emph{variance}, as $f_i(\bm z)=t_i$ for any $\bm z$.

\full{Indeed:
\begin{align}
\label{eq:AWG_vAR}
f_i&(\bm z)=E_{\bm s_i\sim \calY(\bm z,t_i)}[d_E(\bm s_i,\bm z)]
\\
&=E_{\bm s_i\sim \calY(\bm z,t_i)}[\sum_{j\leq m}( s_{ij}- z_j)^2]= VAR_{\bm s_i\sim \calY(\bm z,t_i)}[\bm \eps_i]=t_i.\notag
\end{align}
}{}

\newpar{Aggregation}
Given an instance $\tup{S,\bm z}$, an aggregation function returns predicted labels $\hatvm z$.  We define the  \emph{error}  as $d(\bm z, \hatvm z)$. For example \emph{unweighted aggregation} in the real-valued domain simply returns the mean of workers' answers. The goal of truth discovery is to find algorithms that return labels with low expected error, see Table~\ref{tab:s_example}. 

When the type of every worker is known,  for many noise models there are accurate characterizations of the optimal aggregation functions. For example, the best linear unbiased estimator under the AWG model with NSED is taking the mean of workers' answers, inversely weighted by their variance~\cite{aitkin1935least}. 




\newsec{Fault Estimation}\label{sec:AK}

Our key approach is relying on estimating $f_i$ using the average distance of worker~$i$ from all other workers. Formally, we define $d_{ii'} := d(\bm s_i, \bm s_{i'})$, and the \emph{average pairwise distance} is

\vspace{-0mm}
\begin{equation}\label{eq:pi}
    \pi_i  :=\frac{1}{n-1}\sum_{i'\in N\setminus\{ i\}}d_{ii'}.\end{equation}

Next, we analyze the relation between $\pi_i=\pi_i(S)$ (which is a random variable) and $f_i$, which is an inherent property that is deterministically induced by the worker's type.
For an element $\bm s\in Z$ we consider the induced noise variable $\bm\eps_{\bm s}:=\bm s-\bm z$. We denote by $\tilde\calY(\bm z,t)$ the distribution of $\bm\eps_{\bm s}$ (where $\bm s\sim \calY(\bm z,t)$). Thus under NSED we have that $d(\bm s,\bm z)=\|\bm\eps_{\bm s}\|^2$.

We  define $\bias_i(\bm z) := E_{\bm\eps_i\sim \tilde\calY(\bm z,t_i)}[\bm\eps_i]$  as the  \emph{bias} of a type~$i$ worker, and $\bias_{\calT}(\bm z):=E_{\bm\eps\sim \tilde\calY(\bm z,\calT)}[\bm\eps]$ as the mean bias of the proto-population.   E.g. in Euclidean space $\bias_i(\bm z)$  is a vector where $\b_{ij}(\bm z)>0$ if $i$ tends to overestimate the answer of question $j$, and negative values mean underestimation.  

Our main conceptual result is an approximately linear connection between the expectations of $\pi_i$ and $d(\bm s_i,\bm z)$.
\begin{theorem}[Anna Karenina Principle]\label{thm:AK_general}
\vspace{0mm}
$$E_{S\sim \calY(\bm z,\calT)^n}[\pi_i|t_i,\bm z] = f_i(\bm z) + \mu_{\calT}(\bm z)-2\tup{\bias_i(\bm z),\bias_{\calT}(\bm z)}.$$
\end{theorem}
 We can also see this linear relation in three datasets (with different labels and distance measures) on Fig.~\ref{fig:corr_all_main}.
\begin{figure}[t]
    \centering\vspace{-0mm}
    \includegraphics[width=0.45\textwidth]{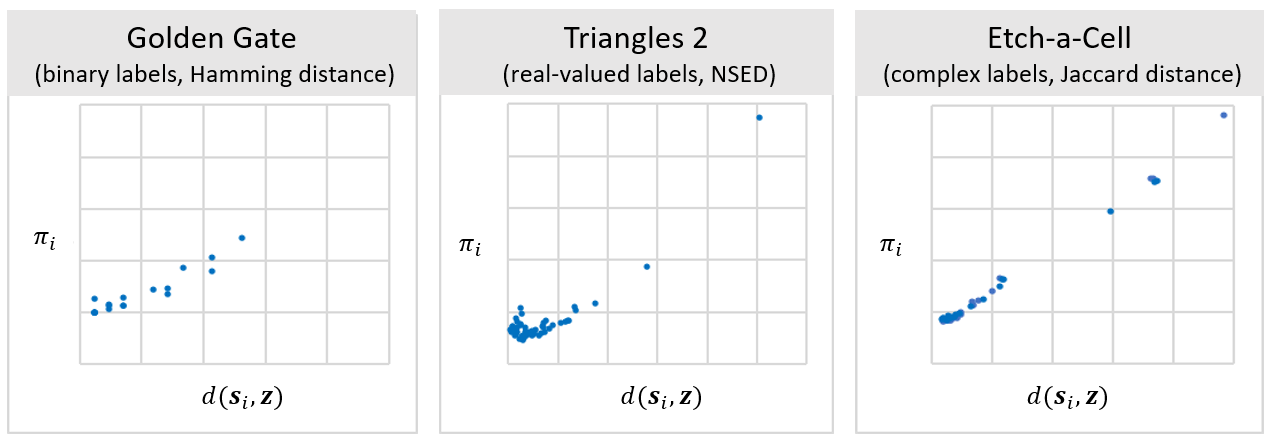}
   \caption{Realized average  distance $\pi_i$ vs. error. Each point is a worker.}
    \label{fig:corr_all_main}
\end{figure}

The proof is rather  straight-forward, and is relegated  to \full{ Appendix~\ref{apx:proofs}, with  examples in Appendix~\ref{apx:empirical_more}}{the full version of the paper}. In particular it shows by direct computation  that the expectation of $d_{ii'}$ for every pair of workers is 
\labeq{dii'}
{E[d_{ii'}|t_i,t_{i'},\bm z]= f_i(\bm z) + f_{i'}(\bm z) - 2 \tup{\bias_i(\bm z),\bias_{i'}(\bm z)}.}

Despite (or perhaps because of) its simplicity, the principle above is highly useful for estimating workers' competence. If  $\pi_i$ is roughly linearly increasing in $f_i$,  a na\"ive  approach to estimate $f_i$ from the data is by setting $\hat f_i$ to be some increasing function of $\pi_i$. 

However there are several obstacles we need to overcome in order to get theoretical guarantees. 

\full{
\begin{description}
    \item[Concentration bounds] I.e., showing that w.h.p the average similarity is close to its expected value, and thus approximates the competence.
    
    \item[Dependency on the ground truth] As long as $E[\pi_i]$ depends on $\bm z$, we have a chicken and egg problem, since we do not know the truth. 
    We will show conditions under which $E[\pi_i]$ does not depend on $\bm z$ at all. 
    \item[Biases] The value depends on the bias $\bias_i$ which is part of the unknown type. Estimating or bounding this bias is a major part of our work.
    \item[Population parameters] Another dependency is on $\mu_{\calT}$ and $\bias_{\calT}$, which depend on the population. These parameters can often be easily guessed or estimated. We will also consider a conservative approach that assigns a default value of $0$. \rmr{remove?}
\end{description}

}{
In particular: concentration bounds; 
estimation of $\mu$; and the biases that appear in the last term; all of which will be tackled in the next sections. 
}

For concreteness, we assume in the remainder of the paper, except where explicitly stated otherwise, that the inner product space $(Z,\tup{})$ is  $\mathbb R^m$, and  $d$ is NSED.

\newpar{Concentration bounds}How far is the empirical average $\pi_i$ from its expectation?
We show that when the noise on all questions is independent and bounded, the probability of a large estimation error decreases linearly with the sample size $\min\{n,m\}$. \full{ For details see Appendix~\ref{apx:sample}.}



\subsec{Domain-independent bounds}\label{sec:no_model}
\rmr{shorten}
What can be said without symmetry or other assumptions on the model? We argue that we can at least tell particularly poor workers from good workers. 

\begin{corollary}\label{cor:separate_GB}
Consider a ``bad" worker $i^*$ with $f_{i^*}>9 \mu_\calT$, and a ``good" worker $i^{**}$ with $f_{i^{**}}<\mu_\calT$. Then $E[\pi_{i^{*}}]>E[\pi_{i^{**}}]$. No better separation is possible (i.e. there is an instance where all the inequalities  become equalities). 
\end{corollary}
The proof relies on the following lemma, which is itself a corollary of the Anna Karenina principle (Thm.~\ref{thm:AK_general}) and the Cauchy-Schwarz inequality: 
\begin{lemma}\label{lemma:pi_LB_UB}
For any worker $i$ and any $\gamma\geq 0$, if $f_i=\gamma \mu$, then $E[\pi_i|t_i]\in (1\pm\sqrt\gamma)^2 \mu$.
\end{lemma}

\begin{proof}[Proof of Corollary~\ref{cor:separate_GB}]
By Lemma.~\ref{lemma:pi_LB_UB}, $E[\pi_{i^{**}}]\leq 4\mu < E[\pi_{i^{*}}]$.

Without further assumptions, this condition is tight. To see why, consider a population on $\mathbb R$ where $z=0$.  The good worker~$i^{**}$ provides the fixed report $s_{i^{**}}=-1$, the poor worker $i^{*}$ provides the fixed report $s_{i^{*}}=3-\delta$. However the measure of types $t_{i^*},t_{i^{**}}$ in $\calT$ is 0, and w.p. 1 type $t'$ is selected with a fixed report $s=1$.
Note that $\mu=(s-z)=1$, and thus $f_{i^{**}}=1=\mu$ whereas $f_{i^*}=9=9\mu$.

However, the reports of $i^*, i^{**}$ are completely symmetric around $1$, in the absence of more workers there is no way to distinguish between these two workers, by their disparity or otherwise. 
\end{proof}

In the special case of interest when there are only two types of workers (a situation known as ``Hammer-spammer"~\cite{karger2011iterative}), Lemma~\ref{lemma:pi_LB_UB} enables us to separate good from bad workers even more easily. This essentially depends on the fraction of bad workers and on their bias. \full{See Appendix~\ref{apx:proofs} for details and proofs.}{}

\newpar{Symmetric Noise}
A trivial implication of Theorem~\ref{thm:AK} is when the average worker is unbiased: 
\begin{corollary}[Anna Karenina principle for zero bias]\label{cor:AK_sym}
If 
 $\bias_\calT\!=\!0$ then $E[\pi_i|t_i]=f_i\!+\!\mu_\calT$ for all $i$.
\end{corollary}
This means that given enough samples, we can  retrieve workers' exact fault  level with high accuracy, by setting $\hat f_i:=\pi_i(S)-\hat \mu$.  This will be important later on when we discuss aggregation. 

\medskip What if we use other distance measures than NSED? Suppose that $d$ is an \emph{arbitrary distance metric} over  space $Z$, $\bm z \in Z$ is the ground truth, and $\bm s_i\in Z$ is the report of worker~$i$. $f_i$ and $\pi_i$ are defined as before. 
Intuitively, we say that the noise model $\calY$ is \emph{symmetric} if  
for every point $\bm x$ there is an equally-likely point that is on ``the other side" of $\bm z$ (note that this in particular implies zero bias).

\begin{theorem}[Anna Karenina principle for symmetric noise and distance metrics]~\label{thm:AK_metric}
If $d$ is any distance metric and $\calY$ is symmetric, then $\max\{\mu,f_i\} \leq E[\pi_i|t_i] \leq \mu+ f_i$.
\end{theorem}


An immediate corollary of Theorem~\ref{thm:AK_metric} is that for poor workers with $f_i\geq \mu$,  the average distance $\pi_i$ is a 2-approximation for $f_i$ (up to noise). See details and proof in \full{Appendix~\ref{apx:metric}.}{the full version.}


\subsec{Domain-specific results}
\newpar{Binary labels}
Kurvers \etal~\shortcite{kurvers2019detect} considered the average similarity of workers when answering a set of yes/no questions, and the type of a worker is her probability $p_i$ to answer correctly independently over each question, a model known as the \emph{one-coin} model or the \emph{Dawid-Skene} model.

They showed that the (expected) average similarity is an increasing linear function  of $p_i$. 

Interestingly, the  result from \cite{kurvers2019detect} can also be obtained directly from Theorem~\ref{thm:AK_general}, by plugging in the  Hamming distance (which is just NSED on the binary cube $\{-1,1\}^d$ instead of $\mathbb R^d$). This result can also be easily extended to multiple-choice labels. For details see \full{Appendix~\ref{apx:binary}.}{the full version.}


\newpar{Cosine similarity}
When  label vectors are normalized, we have that $d(\bm x,\bm y)=2(1-cos(\bm x,\bm y))$, meaning that ranking workers by decreasing average cosine similarity (as suggested in \cite{kobayashi2018frustratingly}) is the same as ranking them by increasing average NSED. Our results above provide sufficient conditions for when this separates good workers from poor ones. 




\rmr{add a paragraph on aggregation? We can just say we weigh agents based on estimated competence. Then add a similar paragraph in the CONT section saying we weigh them inversely to the variance, and showing equivalence and consistency}

	\newsec{AWG Model and Maximum Likelihood}\label{sec:AWG}
\full{By imposing the same variance on all dimensions, the model assumes that questions are equally difficult.\footnote{In the AWG model the equal-difficulty assumption is a normative decision, since we can always scale the data. Essentially, it means that we measure errors in standard deviations, giving the same importance to all questions.}}{}



    Since AWG has no bias, we know from Cor.~\ref{cor:AK_sym} that 
	$E[\pi_i(S)|t_i]  = f_i + \mu$. 
Thus if we have a good estimate $\hat \mu$ of $\mu$, setting $\hat f_i:=\pi_i-\hat \mu$ is a reasonable heuristic. 
In this section we show that under a slight relaxation of the AWG model, tweaking the heuristic above provides the MLE for $f_i$.

\medskip
Denote 
$\bar \mu :=  \frac{1}{2n}\sum_{i\in N}\pi_i(S)$, 
that is,  half the average pairwise distance. Note that for unbiased workers, we have by Cor.~\ref{cor:AK_sym} that $E[\frac1n\sum_{i\in N}\pi_i(S)]=2\mu$, and thus $\bar\mu$ is an unbiased estimator of $\mu$.
We thus set $\hat f^{NP}_i := \pi_i-\bar \mu$, where $NP$ stands for Na\"ive Proxy.

\newpar{Computing the MLE}
From Eq.~\eqref{eq:dii'} and zero-bias, we have that
$E[d_{ii'}|t_i,t_{i'}]=f_i + f_{i'}$. However even under the AWG model, the pairwise distances are correlated. 
For the analysis, we will neglect these correlations, and assume  that the pairwise distances are all independent conditional on workers' types. More formally, under this \emph{`pairwise' Additive White Gaussian} (pAWG)  model, 
$d_{ii'} = f_i+f_{i'} + \eps_{ii'}$, where all of $\eps_{ii'}$ are sampled i.i.d. from  a normal distribution with mean 0 and unknown variance. Ideally, we would like to find $\hatv f:=(\hat f_i)_{i\in N}$ that minimize the estimation errors $(\eps_{ii'})$,\footnote{We  use $\hatv f$ instead of hat+arrow accent.} which is an Ordinary Least Squares (regression) problem. 

We next show how to derive  a closed form solution for the maximum likelihood estimator of the fault $\vec f$, also allowing for a regularization term with coefficient $\lambda$.\lirong{Technically with the regularizer it is not MLE} The theorem might have independent interest as it allows us to estimate a matrix created from adding a vector to itself (an ``outer sum'' matrix) from its off-diagonal entries. 
\begin{theorem}\label{thm:RLS}Let $\lambda\geq 0$, and  $D=(d_{ii'})_{i,i'\in N}$ be an arbitrary symmetric nonnegative matrix. Then
$$\argmin_{\hatv f}\!\!\!\!\!\!\!\!\sum_{i,i'\in N: i\neq i'}\!\!\!\!\!\!(\hat f_i+\hat f_{i'}- d_{ii'})^2 +\lambda \|\hatv f\|^2_2 = \frac{2(n-1)\vec \pi\! -\! \frac{8n(n-1)}{4n+\lambda-4}\bar \mu}{2n+
\lambda-4}.$$
\end{theorem}

Our main technical result in this paper follows as a direct corollary of the above theorem, when the matrix $D$ represents pairwise distances:
\begin{theorem}\label{thm:pAWG_opt}For  $\lambda=4$, 
the regularized maximum likelihood estimator of $\vec f$ in the pAWG model is proportional to $\hatv f^{NP}$. 
\end{theorem}
That is, our heuristic estimate  $\hatv f^{NP}$ is in fact the optimal solution of a regularized pAWG model.

By setting $\lambda=0$ (no regularization) we get a slight variation of this heuristic  $\hatv f^{ML}:=\vec \pi-\frac{n}{n-1}\bar \mu$ (for `maximum likelihood').  

In Fig.~\ref{fig:PTD_compare} we compare implementations of our truth discovery  algorithm (defined in Section~\ref{sec:PTD}, derived with different values of $\lambda$.  As expected, more regularization leads to better performance on a small dataset, whereas the unregularized version is optimal in the limit.
\begin{figure}[t]
    \centering
    \includegraphics[width=8.2cm]{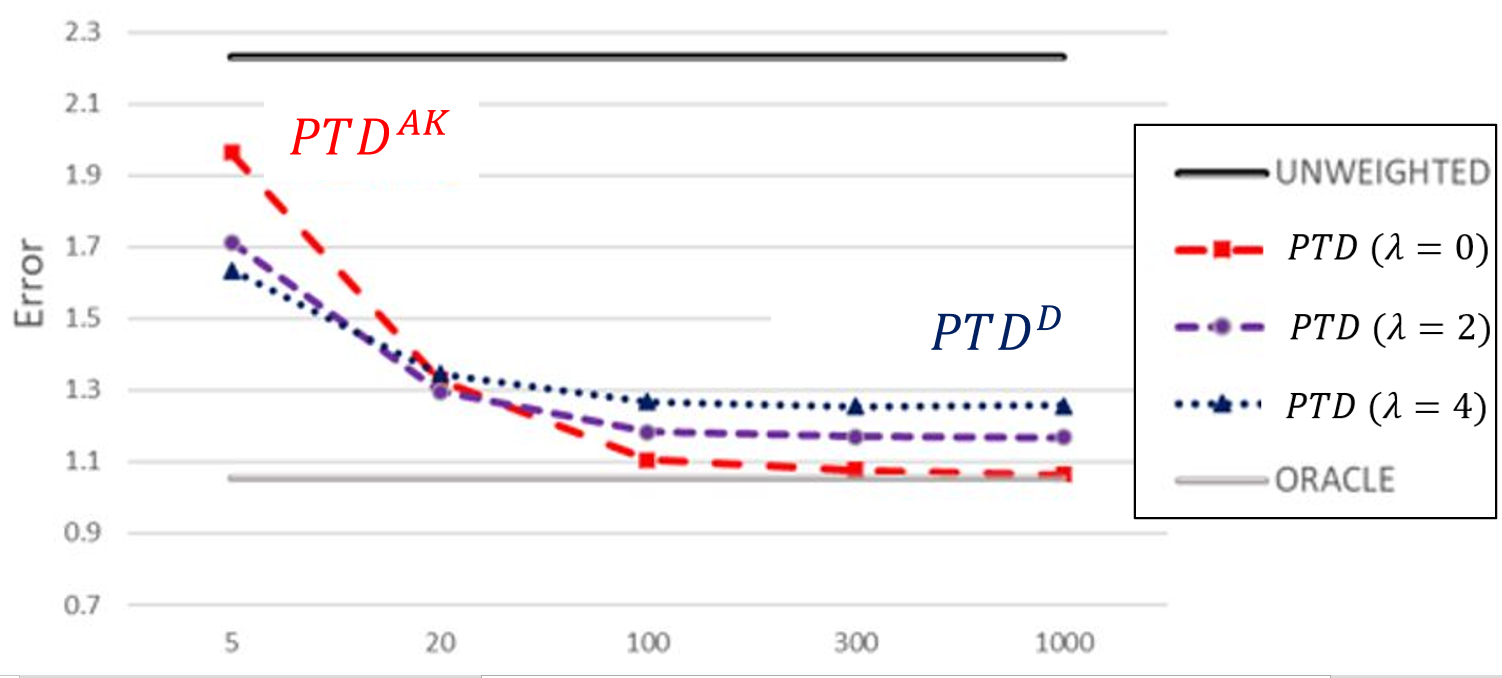}
    \caption{A comparison of the \PTD algorithm variants  on a synthetic real-valued dataset. The x-axis shows the number of questions $m$, whereas the number of workers is fixed ($n=5$). `ORACLE' weighs workers according to their true competence.}
    \label{fig:PTD_compare}
\end{figure}
\rmr{show the graph here? full proof?}


\begin{proof}[Proof of Theorem~\ref{thm:RLS} for $\lambda=0$]
Let $P$ be a list of all $n^2-n$ ordered pairs of $[n]$ (without the main diagonal) in arbitrary order. 
Setting $\lambda=0$, we are left with the following least squares equation:
$$\min_{\hatv f}\sum_{i,i'}(\hat f_i+\hat f_{i'}- d_{ii'})^2 = \min_{\hatv f} \|A\hatv f-\bm d\|_2^2,\vspace{-0mm}$$
where $A$ is a $|P|\times n$ matrix with $a_{ki}=1$ iff $i\in P_k$, and $\bm d$ is a $|P|$-length vector with $d_k = d_{i{i'}}$ for $P_k=(i,{i'})$. 

Fortunately, $A$ has a very specific structure that allows us to obtain the above closed-form solution. Note that every row of $A$ has exactly two `1' entries, in the row index $i$ and column index ${i'}$ of $P_k$; the total number of `1' is $2(n^2-n)$; there are $2n-2$ ones in every column; and every two distinct columns $i,{i'}$ share exactly two non-zero entries (at rows $k$ s.t. $P_k=(i,{i'})$ and $P_k=({i'},i)$). 
This means that $(A^T A)$ has $2n-2$  on the diagonal and $2$ in any other entry.

The optimal solution for ordinary least squares is obtained at $\hatv f$ such that $(A^T A)\hatv f = A^T \bm d$.
By the structure of $A$:
\labeq{ATD}
{[A^T\bm d]_i = 2\sum_{{i'}\neq i} d_{i{i'}} = 2\pi_i,\vspace{-0mm}
}
and
\vspace{-4mm}
\labeq{ATAf}
{[(A^T A)\hatv f]_i =  2(n-2)\hat f_i+2\sum_{i'\in N}\hat f_{i'}.}
Denote 
\vspace{-3mm}
\labeq{alpha1}
{\alpha := \sum_{i\in N} [A^T\bm d]_i = \sum_{i\in N} 2\sum_{{i'}\neq i} d_{i{i'}} =4n(n-1)\bar \mu,\vspace{-0mm}}

then by Eq.~\eqref{eq:ATAf}:
\begin{align}
\alpha&=\sum_{i \in N} [(A^T A)\hatv f]_i 
= \sum_{i\in N}  [2(n-2)\hat f_i+2\sum_{i' \in N} \hat f_{i'}] \label{eq:alpha2}\\
&= 2(n-2)\sum_{i\in N} f_i + 2\sum_{i'\in N} f_{i'} (\sum_{i\in N} 1)
= 4(n-1)\sum_{i\in N}\hat  f_i.\notag
\end{align}

We can now write the $n$ linear equations as
\begin{small}
\begin{align*}
   & 2(n-1)\pi_i\!=\![A^T\bm d]_i \!=\![(A^T A)\hatv f]_i \tag{By Eqs.~\eqref{eq:ATD},\eqref{eq:ATAf}}\\
   &=(2n-4)\hat f_i + 2\sum_{i'\in N}\hat f_{i'}&\iff&\\ 
    & (n-1)\pi_i =  (n-2)\hat f_i+\sum_{i'\in N} \hat f_{i'} &\iff&\\
     &\hat f_i = \frac{\pi_i}{n\!-\!2}- \frac{1}{n\!-\!2}\sum_{i'\in N} \hat f_{i'} = \frac{\pi_i}{n\!-\!2}-\frac{1}{n\!-\!2}\frac{\alpha}{4(n\!-\!1)} 
    \\
     &\stackrel{Eq.~\eqref{eq:alpha1}}{=}   \frac{n-1}{n-2}\pi_i - \frac{n}{n-2}\bar\mu  = \frac{n-1}{n-2}\hat f^{ML}_i,
\end{align*}
\end{small}
as required, since $\frac{n-1}{n-2}$ is a constant.
\end{proof}

\newsec{Aggregation}\label{sec:PTD}


\rmr{write the PTD}

Our Proximity-based Truth Discovery (\PTD) algorithm is a direct adaptation of the Anna Karenina principle.
The idea is very simple: 
\begin{enumerate}
    \item Compute the average distance [or similarity] $\pi_i$ of every worker;
    \item Estimate fault [or competence] $\hatv f$ from $\vec \pi$;
    \item Aggregate answers, giving higher weight to workers with low fault [high competence].
\end{enumerate}
\begin{algorithm}[t]
\SetAlgoNoLine
\caption{\textsc{(\PTDD) for real-valued data}}
\label{alg:PTDD}
\KwIn{Dataset $S\in \mathbb R^{n \times m}$.}
\KwOut{Est. fault levels  $\hatv f\in\mathbb R^n$; answers $\hatvm z\in \mathbb R^m$.}
Compute $d_{ii'} \leftarrow d(\bm s_i,\bm s_{i'})$ for every pair of workers\;
\For {each  worker $i\in N$}{ 
    set $\pi_i \leftarrow \frac{1}{n-1}\sum_{i'\neq i}d_{ii'}$\tcp*{Step~1}
    }
Set $\bar \mu \leftarrow  \frac{1}{2n}\sum_{i\in N}\pi_i$\;
\For{each  worker $i\in N$}{ 
    Set $\hat f_i \leftarrow \pi_i-\frac{n}{n-1}\bar\mu$\tcp*{Step~2}
    Set $w_i \leftarrow \frac{1}{\hat f_i}$\;
}
Set $\hatvm z\leftarrow \frac{\sum_i w_i \bm s_i}{\sum_i w_i}$\tcp*{Step~3}
\Return $(\hatv f, \hatvm z)$\;
\end{algorithm}

Our default implementation (denoted \PTDD) simply sets weights proportional to the estimated competence, which is in turn proportional to the average similarity, as in \cite{kobayashi2018frustratingly,kurvers2019detect}. 

As we make more assumptions on the structure of labels and the statistical model, we can use an appropriate Anna-Karenina theorem to improve  Step~2, resulting in a domain-specific implementation \PTDAK. For example, Alg.~\ref{alg:PTDD} shows the implementation for the real-valued domain, where Step~2 is based on $\hatv f^{ML}$ defined in Sec.~\ref{sec:AWG}, and Step~3 is based on \cite{aitkin1935least}. \full{See Appendix~\ref{apx:algs} for details on other domain-specific implementations.}{} 

Lastly, we can iteratively repeat the process by computing the \emph{weighted} average distance to other workers. This iterative $\PTD$ algorithm is denoted by \IPTD.

\def\altmu{\check{\mu}}

\newsec{Empirical Evaluation}\label{sec:empirical}
    
\newpar{Algorithms}
We compare the predicted label accuracy of our algorithms (\PTDD,\PTDAK,\IPTD) to unweighted aggregation (\UA); to three general-domain algorithms: \MAS~\cite{braylan2020modeling}, \TT and \TTEXP~\cite{kawase2019graph};  and to  domain-specific algorithms: \CRH~\cite{li2014resolving}, \BPTD~\cite{karger2011iterative}, \DSTD~\cite{dawid1979maximum}, \EVTD~\cite{parisi2014ranking}, \CATD~\cite{li2014confidence}, \GTM~\cite{zhao2012probabilistic}, and \KDEm~\cite{wan2016truth}. 

\newpar{Datasets} We used the following datasets from five different domains. We write the used distance measure in each domain in brackets. 
\begin{description}
\item[Categorical (Hamming distance):]
 GG, DOGS, FLAGS~\cite{shah2015double}; Predict~\cite{mandal2020} (we used data from Oct.8, under all four treatments); and
all six categorical datasets from \cite{kawase2019graph}. We used weighted majority for aggregation.
\item[Real-valued (NSED):] BUILDINGS (collected for this paper); TRI~\cite{hart2018statistical}; and EMO~\cite{snow2008cheap}. Answers aggregated using weighted mean.
\item[Ranking (Kendall-tau):]   DOTS  and PUZZ contain subjective rankings of four images of dots / 8-puzzle boards, according to the number of dots they contain / number of steps from solution ~\cite{mao2013better}. We also extracted the ranking information from BUILDINGS.
For aggregation, we used nine different ordinal voting rules, see \full{Appendix}{full version} for details. 
    \item[Language (GLEU):] The TRANSL dataset contains English translations of Japanese sentences~\cite{braylan2020modeling}. The distance measure we used is GLEU, and there is no aggregation (best worker is selected).
    \item[Outlines (Jaccard):] The  Etch-a-Cell dataset contains bitmaps of the outline of a tumor in 2D slices of a cell~\cite{spiers2021deep} (see Fig.~\ref{fig:etch_example}). We use Jaccard distance on the filled shapes, and aggregate labels using pixelwise-majority.
\end{description} 
In addition we generated synthetic datasets using the AWG model (real-valued); the one-coin model (categorical); and Mallows model (ranking). In the HS datasets there are 20\% `hammers'. 

  A detailed description of algorithms' implementation and datasets is in  \full{Appendix~\ref{apx:empirical_more}}{the full version of the paper}. 

To obtain robust results we sampled $n$ workers and $m$ questions without repetition from each dataset (real or synthetic), and repeated the process at least 1000 times for every combination.

\input{TABLES}

\newpar{Evaluation}
The \emph{error} of every algorithm is the distance (as specified above) to the ground truth, averaged over all samples of certain size of a particular dataset. 

In the tables, we compute for each algorithm its \emph{Relative Improvement} $RI(Alg):=\frac{Err(Alg)-Err(\UA)}{Err(Alg)+Err(\UA)}$, where \UA serves as a baseline. Thus RI is in the range $[-1,1]$ where negative numbers mean improvement over $\UA$.

In some cases we see that one algorithm has slightly higher average error (on the graphs) but lower RI, or that the gap in RI is more substantial. This is since the graphs average over instances of varying difficulty, so instances with high baseline error have more effect.

\newpar{Results}
Fig.~\ref{fig:CAT_CONT} and Table~\ref{tab:CAT_CONT} (and more in \full{Appendix~\ref{apx:empirical_more}}{the full version}) show results on categorical and real-valued data, where there are many specialized algorithms. We can see that there is no single `state-of-the-art', as algorithms that do well on some datasets may have poor performance   on other data, or for a different number of voters/questions. 

Yet for moderate $n$ and $m$, all three versions of our \PTD consistently provide good results over almost all datasets, usually beating the three general-domain algorithms, 
and doing roughly at par with the best specialized ones.  We can also see that on synthetic real-valued data with Gaussian noise, the provably-optimal \PTDD is also best in practice.  

On real datasets, our \IPTD is usually better, and as $n$ and $m$ increase sometimes one of the specialized algorithms takes over.  
Intuitively, real datasets may often have a some correlation in poor workers' errors. Iterative algorithms, such as our \IPTD and some of the existing algorithms, are able to overcome this since they gradually rely more on the largest and most consistent set of workers.

The real strength of our approach shows when labels are more complex. Table~\ref{tab:RANK} and Fig.~\ref{fig:RANK} show how our simple algorithms are consistently better than the other three algorithms both on ranking data and on both complex annotation tasks.
Results also show that \PTD yields substantial improvement regardless of the voting rule in use. 
Moreover, while the other algorithms work better on some datasets, they are highly unstable and often perform worse than the baseline.

\newsec{Conclusion}
Average proximity can be used as a general scheme to estimate workers' competence  in a broad range of truth discovery and crowdsourcing scenarios.  Due to the  ``Anna Karenina principle,'' we expect the answers of  competent workers to be much closer to others, than those of incompetent workers, even under very weak assumptions on the domain and the noise model.   Under more explicit assumptions, the average distance accurately  estimates the true competence.  

The above results suggest an extremely simple, general and practical algorithm for truth-discovery (the \PTD algorithm), that weighs workers by their average proximity to others, and can be combined with most  aggregation methods. This is particularly useful in the context of existing crowdsourcing systems where the aggregation rule may be subject to constraints due to legacy, simplicity, explainability, legal, or other considerations (e.g. a voting rule with certain axiomatic properties). In addition, average proximity is simple and flexible enough so we can modify it  to deal with challenges outside the scope of the current paper, such as partial data~\cite{dalvi2013aggregating,karger2011iterative,li2014confidence}; 
semi-supervised learning~\cite{yin2011semi}; or worker's competence that varies across task types~\cite{braylan2020modeling}.

Despite its simplicity, the \PTD algorithm substantially improves the outcome compared to unweighted aggregation. It is also competitive with other, more sophisticated algorithms, especially in the common case of moderate input size. We thus conclude that the average similarity heuristic is indeed a frustratingly easy---and practical---tool for crowdsourcing.



An obvious shortcoming of  \PTD  is  that a group of workers that submit similar labels (e.g. by acting strategically) can boost their own weights. Future work will consider how to identify and/or mitigate the affect of such groups. 

\rmr{CACM papers on crowdsourcing / truth discovery:
https://cacm.acm.org/magazines/2009/12/52830-crowdsourcing-and-the-question-of-expertise/fulltext  (argues there is no ground truth)

https://cacm.acm.org/magazines/2011/2/104390-following-the-crowd/fulltext
}
\section*{Acknowledgements}
This research was supported by the Israel Science Foundation (ISF; Grant No. 2539/20).

Previous versions of this paper have been rejected from 7 (seven!) AI conferences. The current version is much improved thanks to the comments, suggestions, and references provided by the (many) reviewers along the way. 
\FloatBarrier
\begin{small}
 \bibliography{proxy.bib}
 \end{small}
\full{
\clearpage
\appendix

\onecolumn

\input{appendix_AAAI23}
}{}
\end{document}

%% file: defs.tex
\newtheorem{theorem}{Theorem}
\newtheorem{corollary}[theorem]{Corollary}
\newtheorem{lemma}[theorem]{Lemma}

\newtheorem{claim}[theorem]{Claim}

\newtheorem{definition}{Definition}

\newenvironment{rtheorem}[1]{\medskip\noindent\textbf{Theorem~\ref{#1}. }\begin{itshape}}{\end{itshape}\medskip}

\newenvironment{rlemma}[1]{\medskip\noindent\textbf{Lemma~\ref{#1}. }\begin{itshape}}{\end{itshape}\medskip}
\newenvironment{rcorollary}[1]{\noindent\textbf{Corollary~\ref{#1}. }\begin{itshape}}{\end{itshape}}

\newcommand{\sproof}{\noindent{\bf Proof.}\hspace*{1em}}

\def\eps{{\varepsilon}}

\def\literalqed{{\ \nolinebreak\hfill\mbox{\quad}}}
\long\def\symbolfootnote[#1]#2{\begingroup%
\def\thefootnote{\fnsymbol{footnote}}\footnote[#1]{#2}\endgroup} 

\def\calT{{\cal T}}

\def\calX{{\cal X}}

\def\eps{\epsilon}
\DeclareMathOperator{\argmin}{\text{argmin}}

\newcommand{\labeq}[2]{\begin{equation}\label{eq:#1}#2\end{equation}}

\newcommand{\ind}[1]{\llbracket #1 \rrbracket}

\def\({\left(}
\def\){\right)}

\newcommand{\tup}[1]{\left\langle #1 \right\rangle}
\def\ol{\overline}
\renewcommand{\ind}[1]{\left\llbracket #1 \right\rrbracket}
\newcommand{\xqed}{\mbox{\raggedright $\Diamond$}}
\def\xqedhere{\hfill\xqed}

\def\newpar{\vspace{-3mm}\paragraph}

\def\subsec{\subsection}
\def\etal{\emph{et al.}}

\def\cal{\mathcal}

%% file: TABLES.tex
\begin{table*}
\centering
    \begin{small}
    \begin{tabular}{|l|l||l|l|l|l||l|l|l||l|l|l|}
    \hline
Categorical&k&{\CRH}&{\BPTD}&{\DSTD}&{\EVTD}&{\TTEXP}&{\TT}&{\MAS}&{\PTDD}&{\PTDAK}&{\IPTD}\\					\hline
SYN.HS&2&{-14\%}&\aboveUA{+6\%}&\inrange{-23\%}&{-20\%}&{-15\%}&{-13\%}&{-17\%}&{-21\%}&\best{-24\%}&\inrange{-24\%}\\					
SYN.N&2&{-4\%}&\aboveUA{+12\%}&\inrange{-9\%}&{-6\%}&{-7\%}&\inrange{-8\%}&{-5\%}&\inrange{-7\%}&\best{-9\%}&\inrange{-9\%}\\					
GG&2&{-8\%}&{-15\%}&\inrange{-19\%}&{-10\%}&{-14\%}&\best{-20\%}&{-14\%}&{-14\%}&\inrange{-18\%}&{-18\%}\\					
Pred.T1&2&\inrange{-1\%}&\aboveUA{+9\%}&\inrange{-1\%}&\inrange{-1\%}&\inrange{-1\%}&\aboveUA{+0\%}&\inrange{-1\%}&\inrange{-1\%}&\inrange{-1\%}&\best{-1\%}\\					
Pred.T2&2&\inrange{-0\%}&\aboveUA{+10\%}&\aboveUA{+0\%}&\best{-1\%}&\inrange{-1\%}&\aboveUA{+1\%}&\aboveUA{+0\%}&\inrange{-0\%}&\inrange{-1\%}&\inrange{-1\%}\\					
Pred.T3&2&\inrange{-1\%}&\aboveUA{+8\%}&\inrange{-1\%}&\inrange{-1\%}&{-0\%}&\aboveUA{+1\%}&\inrange{-1\%}&\inrange{-2\%}&\inrange{-2\%}&\best{-2\%}\\					
Pred.T4&2&{-1\%}&\aboveUA{+7\%}&{-2\%}&\best{-4\%}&{-1\%}&{-1\%}&{-2\%}&\inrange{-3\%}&\inrange{-3\%}&\inrange{-3\%}\\

					\hline
SYN&4&{-18\%}&{\irrel}&{\irrel}&{\irrel}&{-22\%}&{-29\%}&{-28\%}&{-33\%}&{-33\%}&\best{-41\%}\\					
DOGS&10&{-1\%}&{\irrel}&{\irrel}&{\irrel}&{-2\%}&\aboveUA{+0\%}&\inrange{-3\%}&\inrange{-4\%}&\inrange{-4\%}&\best{-4\%}\\					
FLAGS&4&{-17\%}&{\irrel}&{\irrel}&{\irrel}&{-29\%}&\best{-31\%}&{-27\%}&{-21\%}&{-22\%}&\inrange{-30\%}\\					
Chinese&2&{-2\%}&{\irrel}&{\irrel}&{\irrel}&{-1\%}&\inrange{-3\%}&\inrange{-4\%}&\inrange{-4\%}&\inrange{-4\%}&\best{-5\%}\\					
English&2&\inrange{-1\%}&{\irrel}&{\irrel}&{\irrel}&\inrange{-1\%}&\inrange{-1\%}&\inrange{-1\%}&\inrange{-1\%}&\inrange{-1\%}&\best{-1\%}\\					
IT&4&{-3\%}&{\irrel}&{\irrel}&{\irrel}&{-4\%}&\inrange{-5\%}&\inrange{-6\%}&\inrange{-5\%}&\inrange{-5\%}&\best{-7\%}\\					
Medicine&4&{-3\%}&{\irrel}&{\irrel}&{\irrel}&{-13\%}&\best{-16\%}&{-9\%}&{-8\%}&{-8\%}&{-12\%}\\					
Pokemon&6&{-11\%}&{\irrel}&{\irrel}&{\irrel}&{-25\%}&\best{-27\%}&{-17\%}&{-22\%}&{-22\%}&\inrange{-27\%}\\					
Science&5&\inrange{-1\%}&{\irrel}&{\irrel}&{\irrel}&\aboveUA{+1\%}&\aboveUA{+1\%}&\inrange{-1\%}&\inrange{-3\%}&\inrange{-3\%}&\best{-3\%}\\					
					\hline
Real-valued&&{\CRH}&{\CATD}&{\GTM}&{\KDEm}&{\TTEXP}&{\TT}&{\MAS}&{\PTDD}&{\PTDAK}&{\IPTD}\\					\hline
		SYN.N&\irrel&{-23\%}&{-27\%}&{-14\%}&\aboveUA{+6\%}&{-7\%}&{-11\%}&\aboveUA{+4\%}&{-20\%}&\best{-41\%}&{-25\%}\\					
BUILD&\irrel&{-9\%}&\aboveUA{+5\%}&{-9\%}&\aboveUA{+1\%}&{-8\%}&\aboveUA{+12\%}&\aboveUA{+4\%}&{-7\%}&{-7\%}&\best{-11\%}\\					
TRI1&\irrel&\inrange{-19\%}&{-7\%}&{-14\%}&{-4\%}&{-17\%}&{-3\%}&{-8\%}&{-17\%}&{-10\%}&\best{-20\%}\\					
TRI2&\irrel&{-9\%}&{-6\%}&{-5\%}&{-4\%}&\inrange{-11\%}&\best{-12\%}&\aboveUA{+1\%}&{-7\%}&{-4\%}&\inrange{-11\%}\\					
EMO&\irrel&\aboveUA{+1\%}&\aboveUA{+18\%}&\aboveUA{+1\%}&\aboveUA{+26\%}&\aboveUA{+5\%}&\aboveUA{+22\%}&\aboveUA{+3\%}&\aboveUA{+0\%}&\aboveUA{+5\%}&\aboveUA{+2\%}\\

\hline	
				
    \end{tabular}
    \end{small}    
    \caption{Results (RI) on categorical and real-valued datasets, with $n=10$ workers and $m=15$ questions \full{(more sizes in the appendix)}{}. The best result in each row is  \underline{underlined}, and results that are not statistically different (within 95\% confidence interval in a paired t-test) are marked in \textbf{bold}. Results in {\color{gray}{gray}} are worse than unweighted aggregation. \vspace{-0mm}}
    \label{tab:CAT_CONT}
 \end{table*}

\begin{figure*}[t]
    \centering
    \includegraphics[width=1.0\linewidth]{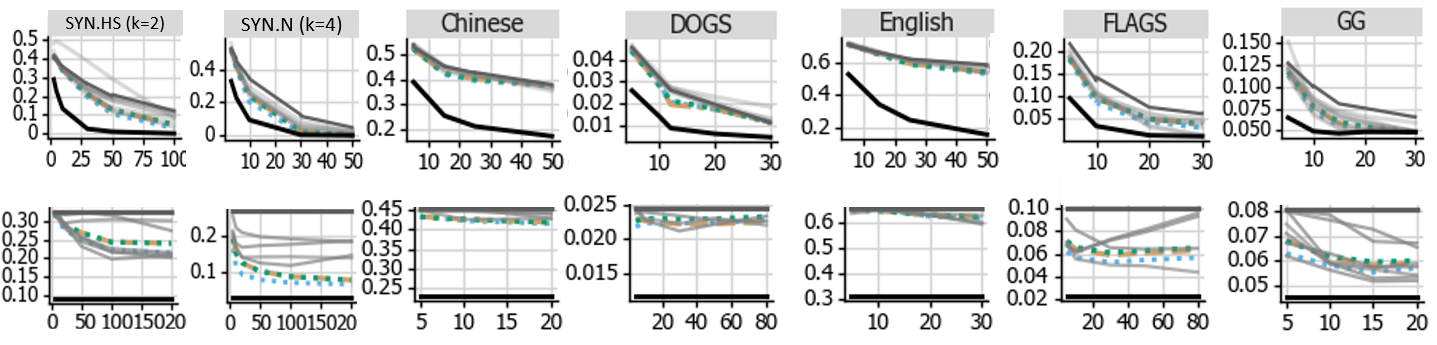}
    \caption{The error on some of the real world datasets as we increase the number of voters $n$ with fixed $m=15$ as in the table (top row), or vary the number of questions $m$ with fixed $n=10$ (bottom row). The black line is the error of an \emph{Oracle} who knows the true fault values $f$ and uses optimal weights. The thin gray lines are all competing algorithms.\vspace{-0mm}}
    \label{fig:CAT_CONT}
\end{figure*}

\begin{figure}[t]
    \centering
    \includegraphics[width=1\linewidth]{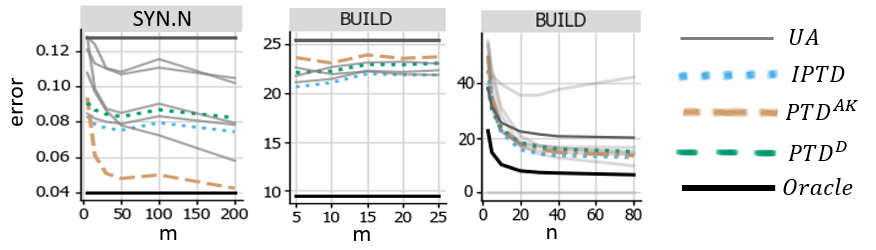}\vspace{-0mm}
    \caption{The error  on real-valued datasets as we vary $m$ (keeping $n=10$) or $n$ (keeping $m=15$).\vspace{-0mm}}
    \label{fig:CONT}
\end{figure}

\begin{table}[t]
    \begin{scriptsize}
    \begin{tabular}{|l|l||l|l|l||l|l|}
    \hline
		Ranking&v. rule&{\TTEXP}&{\TT}&{\MAS}&{\PTDD}&{\IPTD}\\		\hline	
	SYN.HS&Borda&{-5\%}&\best{-24\%}&{-14\%}&{-4\%}&{-10\%}\\			
	SYN.N&Borda&{-3\%}&\aboveUA{+6\%}&\aboveUA{+1\%}&\inrange{-4\%}&\best{-6\%}\\			
	BUILD&Borda&\aboveUA{+0\%}&\aboveUA{+0\%}&\aboveUA{+0\%}&\best{-0\%}&\aboveUA{+0\%}\\			
	DOTS3&Borda&\aboveUA{+2\%}&\aboveUA{+7\%}&\aboveUA{+2\%}&\best{-1\%}&\aboveUA{+1\%}\\			
	DOTS5&Borda&\aboveUA{+1\%}&\aboveUA{+9\%}&\aboveUA{+4\%}&\inrange{-1\%}&\best{-1\%}\\			
	DOTS7&Borda&\aboveUA{+1\%}&\aboveUA{+13\%}&\aboveUA{+5\%}&\best{-3\%}&\inrange{-2\%}\\			
	DOTS9&Borda&{-3\%}&\aboveUA{+7\%}&{-1\%}&{-6\%}&\best{-10\%}\\			
	PUZZ5&Borda&\aboveUA{+1\%}&\aboveUA{+1\%}&\aboveUA{+8\%}&\best{-2\%}&\inrange{-2\%}\\			
	PUZZ7&Borda&{-10\%}&{-7\%}&{-18\%}&{-15\%}&\best{-20\%}\\			
	PUZZ9&Borda&\aboveUA{+6\%}&\aboveUA{+48\%}&\aboveUA{+14\%}&\best{-5\%}&{-1\%}\\			
	PUZZ11&Borda&{-0\%}&\aboveUA{+6\%}&\aboveUA{+3\%}&\best{-3\%}&\inrange{-3\%}\\			
		\hline		
	SYN.N&Plurality&\inrange{-4\%}&{-2\%}&\best{-5\%}&\inrange{-4\%}&\inrange{-5\%}\\			
	BUILD&Plurality&{-1\%}&\aboveUA{+20\%}&{-2\%}&\inrange{-6\%}&\best{-6\%}\\			
	DOTS3&Plurality&\aboveUA{+2\%}&\aboveUA{+14\%}&\aboveUA{+2\%}&\best{-2\%}&\inrange{-1\%}\\			
	PUZZ5&Plurality&\aboveUA{+1\%}&\aboveUA{+12\%}&\aboveUA{+1\%}&\best{-3\%}&\inrange{-2\%}\\			
		\hline		
	SYN.N&Copeland&{-12\%}&\best{-37\%}&{-25\%}&{-27\%}&{-29\%}\\			
	BUILD&Copeland&{-12\%}&{-7\%}&{-17\%}&\best{-27\%}&\inrange{-27\%}\\			
	DOTS3&Copeland&\aboveUA{+0\%}&\aboveUA{+0\%}&\best{-1\%}&\inrange{-0\%}&\inrange{-0\%}\\			
	PUZZ5&Copeland&{-7\%}&{-12\%}&{-16\%}&\best{-19\%}&\inrange{-18\%}\\			
		\hline	\hline	
	Complex&&{}&{}&{}&{}&{}\\	\hline
TRANSL&best v.&{-2\%}&{-3\%}&{-3\%}&\best{-4\%}&\best{-4\%}\\	
ETCH&best v.&{-28\%}&{-39\%}&{-39\%}&\best{-43\%}&\inrange{-42\%}\\	
ETCH&bit. maj.&{-1\%}&{-2\%}&\best{-4\%}&{-2\%}&{-2\%}\\	

\hline
    \end{tabular}
    \end{scriptsize}\caption{Results (RI) for rankings datasets, under three different voting rules ($n=10$, four ranked alternatives), and on the other complex annotation datasets.}    \label{tab:RANK}
    \end{table}
      \begin{figure}[t]
     \includegraphics[width=\linewidth,right]{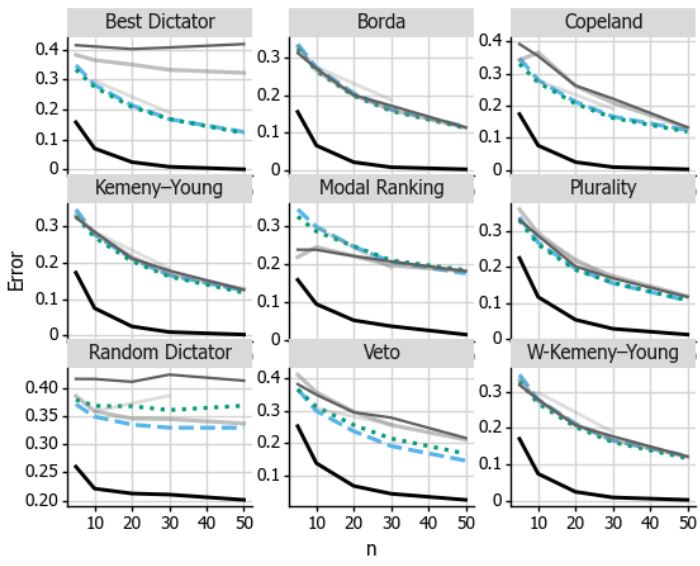}
     \includegraphics[width=0.95\linewidth,right]{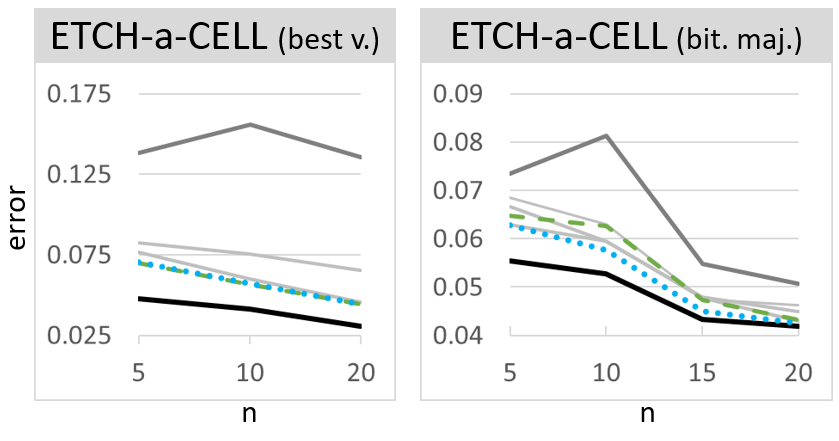}
    \caption{
    Top: error on  DOTS3 under nine different voting rules.  Bottom: error on Etch-a-Cell as number of workers grows.}
    \label{fig:RANK}
\end{figure}

%% file: appendix_AAAI23.tex
We denote by $\ind{C}$ the indicator variable of a true/false condition $C$. 

Throughout the appendix, the  \emph{disparity} of a worker refers to the her average distance to other workers, i.e. $\pi_i$.

 We denote by $\mu_\calT(\bm z):=E_{\bm s\sim \calY(\bm z,\calT)}[d(\bm s,\bm z)]$ the mean fault, omitting $\calT$ and/or $\bm z$ when clear from the context. 
 
\section{A Detailed Graphical Example}\label{apx:example}
Recall our running example from the main text, where we sampled the fault (i.e. the variance) of five workers from $U[50,200]$.

For each worker we independently sampled her answers as $\bm s_i = \bm z + \bm \eps_i$, where  according to the AWG model, each $\eps_{i,j}$ is sampled from a Normal distribution with mean $0$ and variance $t_i$.

The reports of the answers on the for real-valued questions are shown in Table~\ref{tab:s_example_apx} and also graphically in Figure~\ref{fig:s_example}.

For example, to calculate the error of worker~1, we average the squared distance over all four questions/dimensions:
$$d_E(\bm x_1,\bm z) = \frac14((81-80)^2 + (6-0)^2 + (41-40)^2 +(-14 -(-10))^2) = \frac14(1+36+1+16) = 13.5.$$

We can see that workers with lower variance (smaller circles) tend to be closer to the truth but are not necessarily closer. 
The last line shows the Oracle Aggregation, which is the weighted mean with weights inversely proportional to $t_i$. We see that it is indeed closer to the ground truth in this example.  
    

\begin{table}
 \begin{scriptsize}
 \begin{center}
  $ \begin{array}{c|c|cccc|c}
  i ~\backslash~ j&  & 1 & 2 & 3& 4& d(\bm s,\bm z)\\ \hline
  & t_i~\backslash~ z_j & 80 & 0 & 40 & -10 & \\
  \hline
1	&	55	&	81	&	6	&	41	&	-14	&	13.5	\\
2	&	80	&	89	&	-6	&	35	&	4	&	84.5	\\
3	&	100	&	105	&	-18	&	39	&	-5	&	243.7	\\
4	&	120	&	68	&	9	&	62	&	-10	&	177.2	\\
5	&	165	&	67	&	20	&	58	&	-20	&	248.2	\\
			\hline
UA(S)&		&	82	&	2.2	&	47	&	-9	&	14.7	\\
OA(S)&&83.6&0.9&44.3&-8.3& 8.9\\
    \end{array}$
    \end{center}
    \end{scriptsize}\vspace{-0mm}
    
    \caption{\label{tab:s_example_apx}An example of a dataset  sampled from the AWG model. The  bottom row is showing aggregated results using the unweighted mean. The rightmost column shows the \emph{error}, i.e. the distance of every row from the ground truth using NSED. This is the same table as Table~\ref{tab:s_example} in the main text. We also added a row for the \emph{Oracle Aggregation}.  
    }
  \end{table}  
  \begin{figure}[t]
  \centering
 \includegraphics[width=11.1cm]{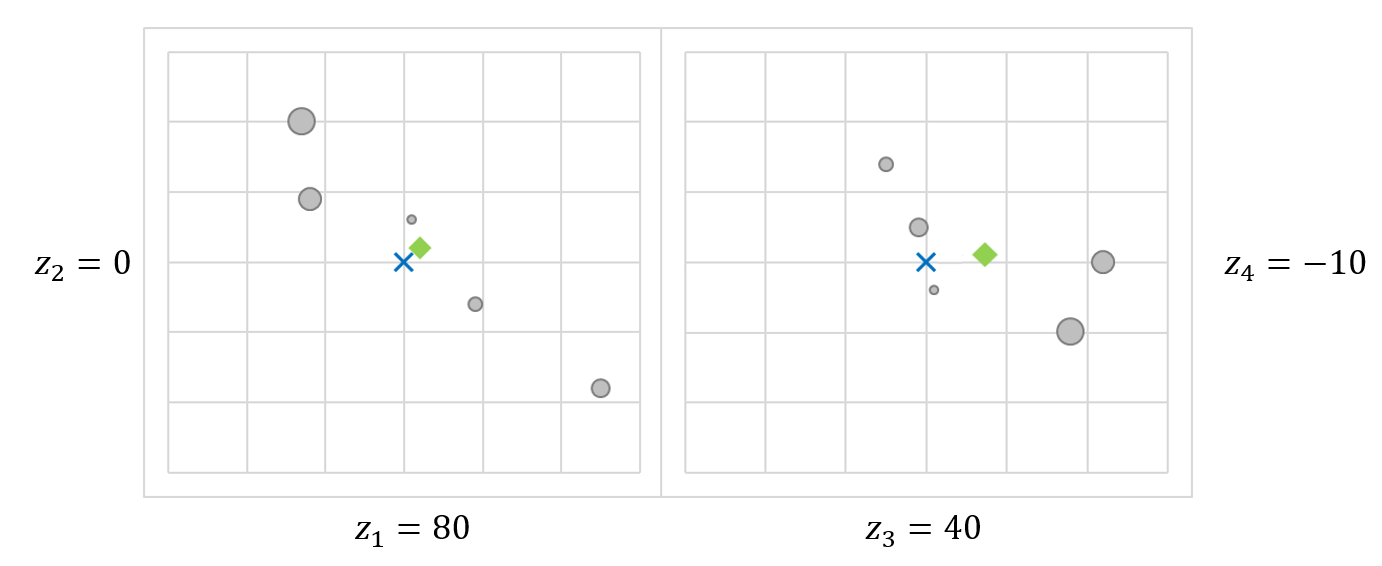}
  \caption{\label{fig:s_example} A graphical representation of Table~\ref{tab:s_example_apx} (for obvious reasons we use two 2-dimensional plots instead of a 4-dimensional one). The blue X marks the ground truth. Workers' reports are marked by gray circles, whose size is proportional to $t_i$ (so smaller circles tend to be closer to the truth). The mean $\UA(S)$ is marked by a green diamond.}
    \vspace{-3mm}
    \end{figure}

\section{General Anna Karenina Theorem}\label{apx:proofs}
Consider an inner product space with a symmetric inner product $\tup{,}$. As usual, the norm $\|\bm x\|$ is a shorthand for $\sqrt{\tup{\bm x,\bm x}}$.

\begin{rtheorem}{thm:AK_general}
$$E_{S\sim \calY(\bm z,\calT)^n}[\pi_i(S)|t_i,\bm z] = f_i(\bm z) + \mu_{\calT}(\bm z)-2\tup{\bias_i(\bm z),\bias_{\calT}(\bm z)}.$$
\end{rtheorem}
\begin{proof}
\begin{align*}
    E_{S\sim \calY(\bm z,\calT)^n}&[\pi_i(S)|t_i,\bm z] = 
    E_{S\sim \calY(\bm z,\calT)^n}[\frac{1}{n-1}\sum_{i'\neq i}\|\bm  s_i - \bm s_{i'}\|^2|t_i,\bm z]\\ 
    &= E_{S\sim \calY(\bm z,\calT)^n}[\frac{1}{n-1}\sum_{i'\neq i}\|\bm  s_i -\bm z - ( \bm s_{i'} - \bm z)\|^2|t_i,\bm z] \\
    &= E_{S\sim \calY(\bm z,\calT)^n}[\frac{1}{n-1}\sum_{i'\neq i}\|\bm  s_i -\bm z\|^2 + \|\bm s_{i'} - \bm z\|^2  - 2\langle\,\bm  s_i -\bm z,\bm s_{i'} - \bm z\rangle|t_i,\bm z] \\ 
    &= E_{\bm s_i \sim \calY(\bm z, t_i)}[\frac{1}{n-1}\sum_{i'\neq i}\|\bm  s_i -\bm z\|^2|t_i,\bm z] + E_{S\sim \calY(\bm z,\calT)^{n-1}}[\frac{1}{n-1}\sum_{i'\neq i}\|\bm  s_{i'} -\bm z\|^2|t_i,\bm z]  \\
    &\qquad \qquad - 2E_{S\sim \calY(\bm z,\calT)^n}[\frac{1}{n-1}\sum_{i'\neq i}\langle\,\bm  s_i -\bm z,\bm s_{i'} - \bm z\rangle|t_i,\bm z] \\
    &= f_i(\bm z) + \frac{1}{n-1}E_{t_{i'} \sim \tau}[\sum_{i'\neq i}E_{s_{i'} \sim \calY(\bm z, t_{i'})}[\|\bm  s_{i'} -\bm z\|^2|t_{i'},\bm z]| \bm z] \\
    &\qquad \qquad - 2\frac{1}{n-1}E_{t_{i} \sim \tau}[\sum_{i'\neq i}E_{{\bm s_i \sim \calY(\bm z, t_i)}, {s_{i'} \sim \calY(\bm z, t_{i'})}}[\langle\,\bm  s_i -\bm z,\bm s_{i'} - \bm z\rangle|t_{i'}, t_i,\bm z]|t_i,\bm z]\\
    &= f_i(\bm z) + \frac{1}{n-1}\sum_{i'\neq i}E_{S\sim \calY(\bm z,\calT)}[f_{i'}| \bm z] \\
    &\qquad \qquad -2\frac{1}{n-1}\sum_{i'\neq i}E_{t_{i} \sim \tau}[\langle\,E_{\bm s_i \sim \calY(\bm z, t_i)}[\bm  s_i -\bm z|t_i,\bm z],E_{\bm s_{i'} \sim \calY(\bm z, t_{i'})}[\bm s_{i'} - \bm z|t_{i'},\bm z]\rangle|t_i,\bm z]\\
    &=f_i(\bm z) + \mu_\calT(\bm z) \\
    &\qquad \qquad - 2\frac{1}{n-1}\sum_{i'\neq i}\langle\,E_{t_{i} \sim \tau}[E_{\bm s_i \sim \calY(\bm z, t_i)}[\bm  s_i -\bm z|t_i,\bm z]|t_i,\bm z],E_{t_{i} \sim \tau}[E_{\bm s_{i'} \sim \calY(\bm z, t_{i'})}[\bm s_{i'} - \bm z|t_{i'},\bm z]|t_i,\bm z]\rangle\\
    &=f_i(\bm z) + \mu_\calT(\bm z) \\
    &\qquad \qquad - 2\frac{1}{n-1}\sum_{i'\neq i}\langle\,E_{\bm s_i \sim \calY(\bm z, t_i)}[\bm  s_i -\bm z|t_i,\bm z],E_{S\sim \calY(\bm z,\calT)}[\bias_{i'}(\bm z)|\bm z]\rangle\\
    &= f_i(\bm z) + \mu_\calT(\bm z) -2\tup{\bias_i(\bm z),\bias_{\calT}(\bm z)}
\end{align*}
\end{proof}

\paragraph{Model-free bounds}

In the following lemmas, all parameters ($\bm b_i, f_i, \mu$) may depend on $\bm z$ as in Theorem~\ref{thm:AK_general}. We omit the $\bm z$ argument for easier reading.

\begin{lemma}\label{lemma:CS}For all $i$,
$|\tup{ \bias_i, \bias_{\calT}}|\leq \sqrt{f_i}\sqrt{\mu}$.
\end{lemma}
\begin{proof}Due to Cauchy-Schwarz inequality,
\begin{align*}
    |\tup{ \bias_i, \bias_{\calT}}|&\leq  \|\bias_i\|\cdot\| \bias_{\calT}\|=   \sqrt{\frac1m\sum_{j=1}^m \b_{ij}}\cdot\sqrt{\frac1m\sum_{j=1}^m \b_{\calT,j}}
    =\sqrt{\frac1m\sum_{j=1}^m (\b_{ij})^2}\cdot\sqrt{\frac1m\sum_{j=1}^m (\b_{\calT,j})^2}\\
        &=\sqrt{\frac1m\sum_{j=1}^m (E_{\bm s_{i}\sim \calY_i}[ s_{ij}-z_j])^2}\cdot\sqrt{\frac1m\sum_{j=1}^m (E_{t\sim \calT}E_{\bm s\sim \calY(t)}[s_{j}-z_j])^2}
\end{align*}
By Jensen inequality and convexity of the square function,
\begin{align*}
    &\leq \sqrt{\frac1m\sum_{j=1}^m E_{\bm s_{i}\sim \calY_i}[ (s_{ij}-z_j)^2]}\cdot\sqrt{\frac1m\sum_{j=1}^m E_{t\sim \calT}E_{\bm s\sim \calY(t)}[(s_{j}-z_j)^2]}\\
    &\leq \sqrt{E_{\bm s_{i}\sim \calY_i}[ \frac1m\sum_{j=1}^m(s_{ij}-z_j)^2]}\cdot\sqrt{ E_{t\sim \calT}E_{\bm s\sim \calY(t)}[\frac1m\sum_{j=1}^m(s_{j}-z_j)^2]}\\
    &\leq \sqrt{E_{\bm s_{i}\sim \calY_i}[ d(\bm s_{i},\bm z)]}\cdot\sqrt{ E_{t\sim \calT}E_{\bm s\sim \calY(t)}[d(\bm s,\bm z)]}
    \leq \sqrt{f_i}\cdot\sqrt{ \mu},\\
\end{align*}
as required.
\end{proof}

The following lemma is another corollary of the Anna Karenina principle. 

\begin{rlemma}{lemma:pi_LB_UB}
For any worker $i$ and any $\alpha\geq 0$, if $f_i=\alpha \mu$, then $E[\pi_i]\in (1\pm\sqrt\alpha)^2 \mu$.
\end{rlemma}
\begin{proof}
For the upper bound, 
\begin{align*}
    E[\pi_{i}] &= f_{i} + \mu -2\tup{ \bias_{i}, \bias_{\calT}}\tag{By Thm.~\ref{thm:AK_general}}\\
    &\leq f_{i} + \mu + 2\sqrt{ f_{i}}\sqrt{\mu}\tag{By Lemma~\ref{lemma:CS}}\\
    &\leq \alpha\mu + \mu + 2\sqrt{\alpha\mu} \sqrt\mu = (1+\alpha+2\sqrt\alpha)\mu = (1+\sqrt\alpha)^2\mu. \tag{since $f_i\leq \alpha\mu$}
\end{align*}
The proof of the lower bound is symmetric. 
\end{proof}

\begin{rcorollary}{cor:separate_GB}
Consider a ``poor" worker $i^*$ with $f_{i^*}>9 \mu_\calT$, and a ``good" worker $i^{**}$ with $f_{i^{**}}<\mu_\calT$. Then $E[\pi_{i^{**}}]<E[\pi_{i^{*}}]$. No better separation is possible.
\end{rcorollary}
\begin{proof}
By Lemma.~\ref{lemma:pi_LB_UB}, $E[\pi_{i^{**}}]\leq 4\mu < E[\pi_{i^{*}}]$.

Without further assumptions, this condition is tight. To see why, consider a population on $\mathbb R$ where $z=0$.  The good worker~$i^{**}$ provides the fixed report $s_{i^{**}}=-1$, the poor worker $i^{*}$ provides the fixed report $s_{i^{*}}=3-\delta$. However the measure of types $t_{i^*},t_{i^{**}}$ in $\calT$ is 0, and w.p. 1 type $t'$ is selected with a fixed report $s=1$.
Note that $\mu=(s-z)=1$, and thus $f_{i^{**}}=1=\mu$ whereas $f_{i^*}=9=9\mu$.

However, the reports of $i^*, i^{**}$ are completely symmetric around $1$, in the absence of more workers there is no way to distinguish between these two workers, by their disparity or otherwise. 
\end{proof}

\begin{corollary}\label{cor:AK_HS}
Suppose there are two types of workers, with $f^P>f^H$.
 For any spammer $i^P$ and hammer $i^{H}$, if \textbf{any} of the following conditions apply,  then $E[\pi_{i^P}] > E[\pi_{i^{H}}]$:
\begin{enumerate}[label=(\emph{\alph*})]
\item  $(b^H)=0$  and $\alpha<\frac{f^P-f^H}{2 (b^P)^2}$;
\item $f^H=0$ and $\alpha<0.5$;
    \item  
    $f^P-f^H>(b^P)^2-(b^H)^2$;
    \item $b^P\leq b^H$.
\end{enumerate}
\end{corollary}Intuitively, this means that it is easy to distinguish hammers from spammers as there are fewer spammers; as the difference in competence becomes larger; and as the spammers are \emph{less} biased.
\begin{proof}
 Denote by $\beta:=\tup{\bm s^P,\bm s^H}$ the amount of agreement in the bias of hammers and spammers. 
 
For every spammer $i^P$
\begin{align}
    E[\pi_{i}|t_i=P]&=f_{i} + \mu - 2\tup{\bias_{i},\bias_{\calT}} \notag\\
    &= f^P + \alpha f^P - 2\tup{\bm s^P,\alpha\bm s^P + (1-\alpha)\bm s^H}\notag\\
    &= f^P + \alpha f^P - \alpha2\tup{\bm s^P,\bm s^P} - (1-\alpha)2\tup{\bm s^P,\bm s^H}\notag\\
    &= f^P + \alpha f^P - 2\alpha (b^P)^2 - (1-\alpha)2\tup{\bm s^P,\bm s^H}\notag\\
    &= f^P + \alpha f^P - 2\alpha (b^P)^2 - 2(1-\alpha)\beta.\label{eq:pi_P}
\end{align}
Similarly, it applies for every hammer $i^H$:
\begin{align}
     E[\pi_{i}|t_i=H]&=f_{i} + \mu - 2\tup{\bias_{i},\bias_{\calT}}\notag\\
    &= f^H + \alpha f^P - 2\tup{\bm s^H,\alpha\bm s^P + (1-\alpha)\bm s^H}\notag\\
    &= f^H + \alpha f^P - 2\alpha\beta - 2(1-\alpha)(b^H)^2.\label{eq:pi_H}
\end{align}

\medskip

(a) Follows immediately from Eqs.~\eqref{eq:pi_P},\eqref{eq:pi_H}, since $\beta=0$.

\medskip
(b) Suppose $f^H=0$, then  Condition~(a) is implied, since  $\frac{f^P-f^H}{2p^2}=\frac{f^P}{2p^2}\geq\frac{(b^P)^2}{2 (b^P)^2}=\frac12$.

\medskip
(c) We note that due to Cauchy-Schwartz inequality,
\labeq{beta_bound}
{|\beta| = |\tup{\bm s^P,\bm s^H}|\leq\|\bm s^P\|\cdot\|\bm s^H\| = (b^P)\cdot (b^H)= \sqrt{(b^P)^2 (b^H)^2}\leq\frac{(b^P)^2+(b^H)^2}{2}. }
Due to Eqs.~\eqref{eq:pi_P},\eqref{eq:pi_H}:
\begin{align*}
    E[\pi_{i^P}] - E[\pi_{i^{H}}] & = f^P + \alpha f^P - 2\alpha p^2 - 2(1-\alpha)\beta - (f^H + \alpha f^P - 2\alpha\beta - 2(1-\alpha)(b^H)^2) \\
    &= f^P-f^H -(2\beta(1-2\alpha) +2\alpha ((b^P)^2+(b^H)^2) - 2h^2)\\
    &\geq f^P-f^H -(((b^P)^2+(b^H)^2)(1-2\alpha) +2\alpha ((b^P)^2+(b^H)^2) - 2h^2)\tag{By Eq.~\eqref{eq:beta_bound}}\\
        &\geq f^P-f^H -(((b^P)^2+(b^H)^2)- 2 (b^H)^2) = f^P-f^H -((b^P)^2-(b^H)^2) >0,\\
\end{align*}
as required.

\medskip
(d) If $(b^P)\leq (b^H)$ then $f^P-f^H>0\geq (b^P)^2-(b^H)^2$ thus Condition~(c) is implied.
\end{proof}

\newsec{Categorical and Binary labels}\label{apx:binary}
Binary labels are perhaps the most common domain where truth discovery is applied.  
In this domain both the ground truth $\bm z$ and workers' answers $\bm s_1,\ldots,\bm s_n$ are given as vectors in $\{0,1\}^m$. 

In order to apply the Anna Karenina principle, we need to show all conditions apply. 


The type of a worker is completely defined by the probability to make an error. Formally, $t_i=(f^*_{ij})_{j\leq m}$, where each $f^*_{ij}$ is a function mapping $z_j\in\{0,1\}$ to an error probability in $[0,1]$, and $Pr_{\bm s_{i}\sim \calY(\bm z,t_i)}[s_{ij}\neq z_j]:= f^*_{ij}(z_j)$.

A special case is the \emph{Independent One-Coin} model\footnote{The One-Coin model is also known as the Dawid-Skene model~\cite{dawid1979maximum}. The `Independent' part refers to the first sampling step where workers' types are sampled i.i.d. from the proto-population.} where $f^*_{ij}(z_j=1)=f^*_{ij}(z_j=0)$ and is the same for all questions $j$ (but possibly varies between workers). 

More general noise models may assign asymmetric error probabilities, with distinct sensitivity and specificity. The result may not hold in this case.

\begin{corollary}[Anna Karenina principle for the binary One-Coin model]\label{cor:AK_bin}
For the binary domain with the One-Coin model, $$E[\pi_i(S_i)|t_i,\bm z] = \mu_\calT(\bm z) + (1-2\mu_\calT(\bm z))f_i(\bm z).$$
\end{corollary}
\begin{proof}
Since $\calX=\{0,1\}^m$, the square Euclidean distance coincides with the normalized Hamming distance, as for any $\bm x,\bm y\in \{0,1\}^m$
$$d_E(\bm x,\bm y)=\frac1m\sum_{j\leq m}(x_j-y_j)^2 =\frac1m|\{j:x_j\neq y_j\}|.$$

For any worker $i$, we have that 
$$f_i(\bm z) = E[d(\bm s_i,\bm z)]=\sum_{j=1}^mPr[s_{ij}\neq z_j]=\frac1m\sum_{j=1}^m f^*_{ij} =\frac1m\sum_{j=1}^m f_{i}^* = f^*_i,$$
and
$$\b_{ij}(\bm z)=E_{\bm s_i\sim \calY(\bm z,t_i)}[s_{ij}-z_j]=\left\{\begin{array}{lll}
Pr_{\bm s_i\sim \calY(\bm z,t_i)}[\bm s_{ij}=0]\cdot 1 &= f^*_{ij} = f^*_i, &\text{ if } z_j=1\\
Pr_{\bm s_i\sim \calY(\bm z,t_i)}[\bm s_{ij}=1]\cdot (-1) &= -f^*_{ij} = -f^*_i, &\text{ if } z_j=0\\
\end{array}
\right.
.$$
In particular, both of $f_i$ and $\tup{\bias_i,\bias_{i'}}$ do not depend on $\bm z$ so we can omit the argument $\bm z$ from the equation.

Also note that $\b_{\calT,j} = E_{t_i\sim \calT}[f_{i}]=\mu_\calT.$
Putting everything together, we get from Theorem~\ref{thm:AK} that
\begin{align*}
    E[\pi_i|t_i] &= f_i+\mu_\calT - \frac2m\tup{\bias_i,\bias_{\calT}} =  f_i+\mu_\calT - \frac2m\sum_{j=1}^m f_{ij} \b_{\calT,j} \\
    &=  f_i+\mu_\calT - \frac2m\sum_{j=1}^m f_{i} \mu
    =  f_i+\mu_\calT - 2f_i\mu_\calT = \mu_\calT + (1-2\mu_\calT)f_i
\end{align*}
\end{proof}

A word is in place regarding the ``bias'' of a worker. Indeed, in the binary domain, all workers that err on a question necessarily err in the same direction (as there is only one wrong answer). Thus the error probability $f_i\leq 0.5$ reflects both the ``noisyness" the bias of the worker. As $f_i$ grows beyond $0.5$, the noise is decreasing but the bias keeps increasing. If the average fault $\mu$ is above $0.5$ then the relation between fault level and disparity becomes negative.

\subsection{Comparison with Kurvers et al.~\shortcite{kurvers2019detect}}
Kurvers et al.~\shortcite{kurvers2019detect} adopt the One-Coin model, without the first sampling stage. That is, they start with a finite population of $N$ workers, each with a given probability $p_i$ to answer correctly (independently for each question). Clearly $p_i=1-f_i$ where $f_i$ is worker $i$'s fault in our model.

They then define the \emph{average decision similarity} $\ol \phi_i$ as the  average percentage agreement of  individual $i$ with all other  individuals. Again it is immediate to see that $\ol \phi_i=1-\pi_i$, as $\pi_i$ was defined in our model as the average \emph{disagreement}.

The main theoretical result in \cite{kurvers2019detect} is that
$$E[\ol \phi_i|p_i]=\frac{1}{N-1}\sum_{j\neq i}(p_i p_j+(1-p_i)(1-p_j)).$$

If we relax the assumption of a finite population by taking $N$ to infinity, we get for every worker $i$:
\labeq{kuv}
{E[\ol \phi_i | p_i]=p_i E[p_j] + (1-p_i)(1-E[p_j]),}
where the expectations on the right side are over the sampling of individuals from the population, and the expectation on the left is over both sampling of workers and answers.

Now it is straightforward to see that Eq.~\eqref{eq:kuv} follows immediately from (in fact equivalent to) our Cor.~\ref{cor:AK_bin}, 
as 
\begin{align*}
    E[\ol \phi_i|p_i] &= p_i E[p_j] + (1-p_i)(1-E[p_j])& \iff\\
1-E[\pi_i|f_i] &= (1-f_i)E[1-f_j]+f_i E[f_j] = (1-f_i)(1-\mu)+f_i\mu&\iff\\
1-E[\pi_i|f_i] &= 1-f_i - \mu +f_i \mu + f_i\mu = 1-(f_i+\mu-2f_i\mu) &\iff\\
E[\pi_i|f_i] &= f_i+\mu - 2f_i\mu=\mu + (1-2\mu)f_i.
\end{align*}

While technically the proof is rather simple under either approach, the advantage of our approach is by showing that there is nothing special about binary labels. In fact, the proof readily extends to multiple labels (as long as errors remain independent!), and the general Anna Karenina theorem shows us how the expression is modified in the presence of dependencies. 

\subsection{Multi-label categorical answers}
Our corollary for binary labels generalizes naturally for multiple labels. Note that in general, the noise model for $k$ labels is described by $k 
\times k$ \emph{confusion matrix}. For general confusion matrices, a weighted plurality may not provide the optimal aggregation even when the matrix of each worker is known~\cite{ben2001optimal}.
However a natural generalization of the  one-coin model is to assume that all $k-1$ wrong answers are equally likely. Under this assumption, we get:
\begin{corollary}[Anna Karenina principle for the multi-label One-Coin model]\label{cor:AK_cat}
For the categorical domain with the One-Coin model, $$E[\pi_i(S_i)|t_i,\bm z] = \mu_\calT(\bm z) + (1-\frac{k}{k-1}\mu_\calT(\bm z))f_i(\bm z).$$
\end{corollary}
The proof is similar to the binary case.

\section{Distance Metrics}\label{apx:metric}
Suppose that $d$ is an arbitrary distance metric over some space $Z$. $\bm z \in Z$ is the ground truth, and $\bm s_i\in Z$ is the report of worker~$i$. The fault level and disparity are defined as the expected distance to the ground  truth  and the actual average distance to the other workers, respectively (just as in the real-valued and binary domains).

\begin{definition}[Symmetric noise]
We say that the noise model $\calY$ is \emph{symmetric} if 
there is an invertible mapping $\phi:Z\rightarrow Z$, such that for all $x$, $d(\bm x,\phi(\bm x))=2d(\bm x,\bm z)$, and for all $t$ and $X\subseteq Z$, 
$Pr_{\bm s\sim \calY(\bm z,t)}[\bm s \in X] = Pr_{\bm s\sim \calY(\bm z,t)}[\bm s=\phi(X)]$. 

\end{definition}
In words, 
for every point $\bm x$ there is an equally-likely point $\phi(\bm x)$ that is on ``the other side" of $\bm z$.
The AWG model is symmetric by this definition  under any $L_p$ norm, by setting $\phi(\bm x) = 2\bm z-\bm x$. Note that this is a stronger symmetry  requirement than zero mean bias, and in particular it entails $\bias_i=0$ \emph{for every worker}. 

\begin{rtheorem}{thm:AK_metric}
If $d$ is any distance metric and $\calY$ is symmetric, then $\max\{\mu_\calT(\bm z),f_i(\bm z)\} \leq E[\pi_i(S)|t_i,\bm z] \leq \mu_\calT(\bm z) + f_i(\bm z)$.
\end{rtheorem}
\begin{proof}
We start with the upper bound, which does not require the symmetry assumption.
\begin{align*}
E[\pi_i(S)|t_i,\bm z] & = E[\frac{1}{n-1} \sum_{i'\neq i}d(\bm s_i,\bm s_{i'})|t_i,\bm z] \\
&\leq E[\frac{1}{n-1} \sum_{i'\neq i}d(\bm s_i,\bm z)+d(\bm z,\bm s_{i'})|t_i,\bm z] \tag{triangle inequality}\\ 
&=E[\frac{1}{n-1} \sum_{i'\neq i}d(\bm s_i,\bm z)|t_i,\bm z] +E[\frac{1}{n-1} \sum_{i'\neq i}d(\bm z,\bm s_{i'})|t_i,\bm z]\\
&=E[d(\bm s_i,\bm z)|t_i,\bm z] +\frac{1}{n-1} \sum_{i'\neq i} E[d(\bm z,\bm s_{i'})|\bm z]\\
&=f_i(\bm z) +\frac{1}{n-1} \sum_{i'\neq i} E_{t_{i'}\sim \calT}E_{\bm s_{i'}\sim \calY(t_{i'},\bm z)}[d(\bm z,\bm s_{i'}(|t_i,\bm z]\\
&=f_i(\bm z) +\frac{1}{n-1} \sum_{i'\neq i} E_{t_{i'}\sim \calT}[f_{i'}(\bm z)]\\
&=f_i(\bm z) +\frac{1}{n-1} \sum_{i'\neq i} \mu_\calT(\bm z) = f_i(\bm z) + \mu_\calT(\bm z).
\end{align*}

For the lower bound, we denote for every point $\bm x\in X$ its symmetric location by $\hat x=\phi(\bm x)$.
Fix some location $\bm s_i$ and some other agent $i'$.
\begin{align*}
E[d(\bm s_i,\bm s_{i'})|\bm s_i,t_{i'},\bm z] &= E_{\bm s_{i'}\sim \calY(t_{i'},\bm z)}[d(\bm s_i,\bm s_{i'})] \\
&= E_{\bm s_{i'}\sim \calY(t_{i'},\bm z)}[\frac12d(\bm s_i,\bm s_{i'}) + \frac12d(\bm s_i,\bias_{i'})] \tag{symmetry} \\
&= \frac12 E_{\bm s_{i'}\sim \calY(t_{i'},\bm z)}[d(\bm s_i,\bm s_{i'}) +  d(\bm s_i,\bias_{i'})] \\
&\geq \frac12  E_{\bm s_{i'}\sim \calY(t_{i'},\bm z)}[  d(\bm s_{i'},\bias_{i'})] \tag{triangle inequality}\\
&=\frac12  E_{\bm s_{i'}\sim \calY(t_{i'},\bm z)}[  2d(\bm s_{i'},\bm z)] \tag{definition of $\bias_{i'}$}\\
& = f_{i'}(\bm z).
\end{align*}

Similarly, for any point $\bm x$,
\begin{align*}
E[d(\bm s_i,\bm x)|t_{i},\bm z] &= E_{\bm s_{i}\sim \calY(t_{i},\bm z)}[d(\bm s_i,\bm x)] \\
&= E_{\bm s_{i}\sim \calY(t_{i},\bm z)}[\frac12d(\bm x,\bm s_{i}) + \frac12d(\bm x,\bias_{i})] \tag{symmetry} \\
&= \frac12 E_{\bm s_i\sim \calY(t_{i},\bm z)}[d(\bm x,\bm s_{i}) +  d(\bm x,\bias_{i})] \\
&\geq \frac12  E_{\bm s_{i}\sim \calY(t_{i},\bm z)}[  d(\bm s_{i},\bias_{i})] \tag{triangle inequality}\\
&=E_{\bm s_{i}\sim \calY(t_{i},\bm z)}[  d(\bm s_{i},\bm z)]  = f_{i}(\bm z).
\end{align*}
That is, the expected distance of $i$ from \emph{any} point is at least $f_i$. Therefore $E[d(\bm s_i,\bm s_{i'})|t_i,t_{i'}]\geq \max\{f_i(\bm z),f_{i'}(\bm z)\}$.
Now, it follows that 
$$E[\pi_i(S)|t_i,\bm z]  = E_{t_{i'}}[d(\bm s_i,\bm s_{i'})|t_i,t_{i'},\bm z] \geq E_{t_{i'}\sim \calT}[\max\{f_i(\bm z),f_{i'}(\bm z)\}] \geq \max\{f_i(\bm z),\mu_\calT(\bm z)\},$$
as required.
\end{proof}
The proof is actually showing a more informative lower bound: $E_{t_{i'}\sim \calT}[\max\{f_i,f_{i'}\}]$. For the best worker this would equal $\mu$, whereas for the worst worker this would equal $f_i$.  

Note that taking $\hat f_i=\pi_i-\hat \mu$ is a good heuristic (assuming $\hat\mu$ is a good estimate of $\mu$), in particular for very good or very bad workers: if $f_i\cong 0$ then due to our upper bound,  
$$\hat f_i = \pi_i-\hat \mu \leq f_i +\mu-\hat \mu \cong f_i.$$
If $f_i \gg \mu$ then due to the lower bound,
$$\hat f_i = \pi_i-\hat\mu \geq f_i - \hat\mu \cong f_i-\mu = (1-o(1))f_i.$$

We can also use this result to show that $\pi$ provides a 2-approximation of $f_i$, at least for the worse-than-average workers. 
\begin{corollary}
If $d$ is any distance metric and $\calY$ is symmetric, then  for any $i$ s.t. $f_i\geq \mu$, we have then $f_i \in [E[\pi_i]/2,E[\pi_i]]$.
\end{corollary}
This is since $f_i= \max\{\mu,f_i\}\leq E[\pi_i]$ on one hand, and $E[\pi_i]\leq f_i+\mu \leq 2 f_i$ on the other hand.  

To see that this is tight (for an Euclidean distance in one dimension), suppose that $\bm z=0$. consider $n-1$ workers with $\bm s_i$ sampled uniformly from $\{-1,1\}$, and another worker $n$ where $\bm s_n=-n$ w.p. $\frac1{2n}$, $\bm s_n=n$ w.p. $\frac1{2n}$, and otherwise $\bm s_n=0$. Then $f_i=\mu=1$ for all workers, but $E[\pi_1]=1+o(1)$ and $E[\pi_n]=2-o(1)$.

A slight modification of the above example also shows why the mean distance alone cannot fully distinct between two workers with $f=2f'$: if we change the support of worker~$n$ above to $\{-n/2,0,n/2\}$ (without changing the probabilities) then $f_n=\mu/2=f_1/2$, yet $E[\pi_1],E[\pi_n]$ both equal $1+o(1)$.

\section{Optimal estimation for the pAWG model}\label{apx:cont}
Suppose that the pairwise distances are all independent conditional on workers' types. More formally, under the \emph{pairwise Additive White Gaussian} (pAWG) noise model, 
$d_{ii'} = f_i+f_{i'} + \eps_{ii'}$, where all of $\eps_{ii'}$ are sampled i.i.d. from  a normal distribution with mean 0 and unknown variance. 


Estimating $\vec f$ from all the off-diagonal entries of the pairwise distance matrix $D$  can be done by solving the following ordinary least squares (regression) problem:\footnote{Using an OLS for this problem is inspired by the work of \cite{parisi2014ranking}, although they do not provide a closed form solution as we do here, nor do they consider regularization.}
$$\min_{\hatv f}\sum_{i,i'}(\hat f_i+\hat f_{i'}- d_{ii'})^2.$$

We will show a closed form solution to a somewhat more general claim, that also allows for a regularization term.

\begin{rtheorem}{thm:RLS}Let $\lambda\geq 0$, and  $D=(d_{i,i'})_{i,i'\in N}$ be an arbitrary symmetric nonnegative matrix (representing pairwise distances). Then
$$\argmin_{\hatv f}\sum_{i,i'\in N: i\neq i'}(\hat f_i+\hat f_{i'}- d_{ii'})^2 +\lambda \sum_{i\in N}(\hat f_i)^2 = \frac{2(n-1)\vec \pi - \frac{8n(n-1)}{4n+\lambda-4}\ol \mu}{2n+
\lambda-4}.$$
\end{rtheorem}


\begin{proof} 
Let $P$ be a list of all $n^2-n$ ordered pairs of $[n]$ (without the main diagonal) in arbitrary order. We can now rewrite the least squares problem as 
 $$\min_{\hatv f}\sum_{(i,i')\in P}(\hat f_i+\hat f_{i'}- d_{ii'})^2 +\lambda \sum_{i\in N}(\hat f_i)^2= \min_{\hatv f} \left(\|A\hatv f-\bm d\|_2^2 + \lambda  \|\hatv f\|_2^2\right),$$
 
 where $A$ is a $|P|\times n$ matrix with $a_{ki}=1$ iff $i\in P_k$, and $\bm d$ is a $|P|$-length vector with $d_k = d_{i{i'}}$ for $P_k=(i,{i'})$. 

The optimal solution for regularized least squares is obtained at $\hatv f$ such that
\labeq{RLS}{(A^T A + \lambda  I)\hatv f = A^T \bm d.}

\begin{figure}
    \centering
    $$P=\begin{array}{c}
         1,2  \\
         1,3\\
         1,4\\
         1,5\\
         2,3\\
         2,4\\
         2,5\\
         3,4\\
         3,5\\
         4,5\\
         5,4\\
         \vdots\\
         2,1
    \end{array};~~~
    A=\left(\begin{array}{ccccc}
         1&1&0&0&0  \\
         1&0&1&0&0  \\
         1&0&0&1&0  \\
         1&0&0&0&1  \\
         0&1&1&0&0  \\
         0&1&0&1&0  \\
         0&1&0&0&1  \\
         0&0&1&1&0  \\
         0&1&1&0&1  \\
         0&0&0&1&1  \\
         \hdashline
              0&0&0&1&1  \\  
              \vdots&&&\vdots\\
                       1&1&0&0&0  \\
    \end{array}\right);~~~ A^TA = \left(\begin{array}{ccccc}
         8&2&2&2&2  \\
         2&8&2&2&2  \\
         2&2&8&2&2  \\
         2&2&2&8&2  \\
         2&2&2&2&8  \\
    \end{array}\right)$$
    \caption{The matrices $A$ and $A^TA$ for $n=5$. The part below the dashed line is a reflection of the part above}
    \label{fig:A_example}
\end{figure}
Fortunately, $A$ has a very specific structure that allows us to obtain the above closed-form solution. Note that every row of $A$ has exactly two `1' entries, in the row index $i$ and column index ${i'}$ of $P_k$; the total number of `1' is $2(n^2-n)$; there are $2n-2$ ones in every column; and every two distinct columns $i,{i'}$ share exactly two non-zero entries (at rows $k$ s.t. $P_k=(i,{i'})$ and $P_k=({i'},i)$). 
This means that $(A^T A)$ has $2n-2$ everywhere on the main diagonal and $2$ everywhere else. See Fig.~\ref{fig:A_example} for a visual example.

 Thus 
\labeq{Ad}
{[A^T\bm d]_i = 2\sum_{{i'}\neq i} d_{i{i'}} = 2(n-1)\pi_i,}
and
\begin{align}
    [(A^T A+\lambda I)\hatv f]_i &= (2n-2)\hat f_i + 2\sum_{{i'}\neq i}\hat f_{i'} +\lambda\hat f_i \notag \\
    &=(2n-2+\lambda)\hat f_i+2\sum_{i'\neq i}\hat f_{i'}= (2n-4+\lambda)\hat f_i+2\sum_{i'\in N}\hat f_{i'}.\label{eq:AAf}
\end{align}
Denote 
\labeq{app_alpha1}
{\alpha := \sum_{i\in N} [A^T\bm d]_i = \sum_{i\in N} 2(2(n-1)\pi_i) =4n(n-1)\ol \mu,}

then 
\begin{align}
\alpha&= \sum_{i
\in N}[A^T\bm d]_i=\sum_{i \in N} [(A^T A+\lambda I)\hatv f]_i 
\tag{By Eq.~\eqref{eq:RLS}}\\
&=\sum_{i\in N}  [(2n-4+\lambda)\hat f_i+2\sum_{i' \in N} \hat f_{i'}] \tag{By Eq.~\eqref{eq:AAf}}\\
&= (2n-4+\lambda)\sum_{i\in N}\hat  f_i + 2\sum_{i'\in N}\hat  f_{i'} (\sum_{i\in N} 1)\notag\\
&= (2n-4+\lambda)\sum_{i\in N} \hat f_i + 2\sum_{i'\in N}\hat  f_{i'} n
= (4n+\lambda-4)\sum_{i\in N}\hat  f_i.\label{eq:app_alpha2}
\end{align}

We can now write the $n$ linear equations specifying $\hatv f$:
\begin{align*}
    &2(n-1)\pi_i=[A^T\bm d]_i =[(A^T A+\lambda I)\hatv f]_i\tag{By Eqs.~\eqref{eq:RLS},\eqref{eq:Ad}}\\
    &=(2n-4+\lambda)\hat f_i + 2\sum_{i'\in N}\hat f_{i'} &\Rightarrow\tag{By Eq.~\eqref{eq:AAf}}\\
    & \hat f_i = \frac{2(n-1)\pi_i-2\sum_{i'\in N} \hat f_{i'}}{2n-4+
    \lambda}\tag{isolating $\hat f_i$} \\
    &= \frac{2(n-1)\pi_i-2\frac{4n(n-1)}{4n-4+\lambda}\ol \mu}{2n-4+
    \lambda},  \tag{by Eqs.\eqref{eq:app_alpha1},\eqref{eq:app_alpha2}}\\
\end{align*}
as required.
\end{proof}

We can use $\hatv f$ obtained for different values of the regularizing constant $\lambda$ in the \PTD algorithm.

\begin{itemize}
    \item By setting $\lambda=0$ (no regularization), we get
$$\hatv f^{ML}=\frac{2(n-1)\vec \pi - \frac{8n(n-1)}{4n-4}\ol \mu}{2n-4} = \frac{(n-1)\vec \pi - n\ol \mu}{n-2} \propto \vec\pi - \frac{n}{n-1}\ol\mu.$$
\item  By setting $\lambda=2$ we get 
$$\hatv f^{ML(\lambda=2)}=\frac{2(n-1)\vec \pi - \frac{8n(n-1)}{4n-2}\ol \mu}{2n
-2}=\vec \pi - \frac{2n}{2n-1}\ol\mu.$$
\item  By setting $\lambda=4$ we get 
$$\hatv f^{ML(\lambda=4)}=\frac{2(n-1)\vec \pi - \frac{8n(n-1)}{4n}\ol \mu}{2n
}=\frac{n-1}{n}(\vec \pi - \ol\mu)  \propto \hatv f^{NP}.$$
\end{itemize}

\rmr{makes sense to regularize using $\|\hatv f-\vec 1\|_2^2$, then getting for $
\lambda=4$,
$$\hatv f^{MLR}=\frac{(n-1)(\bm 
\pi - \ol \mu)+1}{n}.$$
}

Recall that $\lambda=4$ gives us the na\"ive \PTD algorithm. We compare the results empirically in Fig.~\ref{fig:PTD_types}. ``Oracle" is a weighted mean that uses the true types of workers to set the weights (see \OA in Appendix~\ref{apx:empirical_more}). 

We can see that while the na\"ive variant performs better when the pairwise distances are noisy (low $m$), as $m$ grows the $ML$ variant converges to the performance of the oracle. 
\begin{figure}
    \centering
    \includegraphics[width=\textwidth]{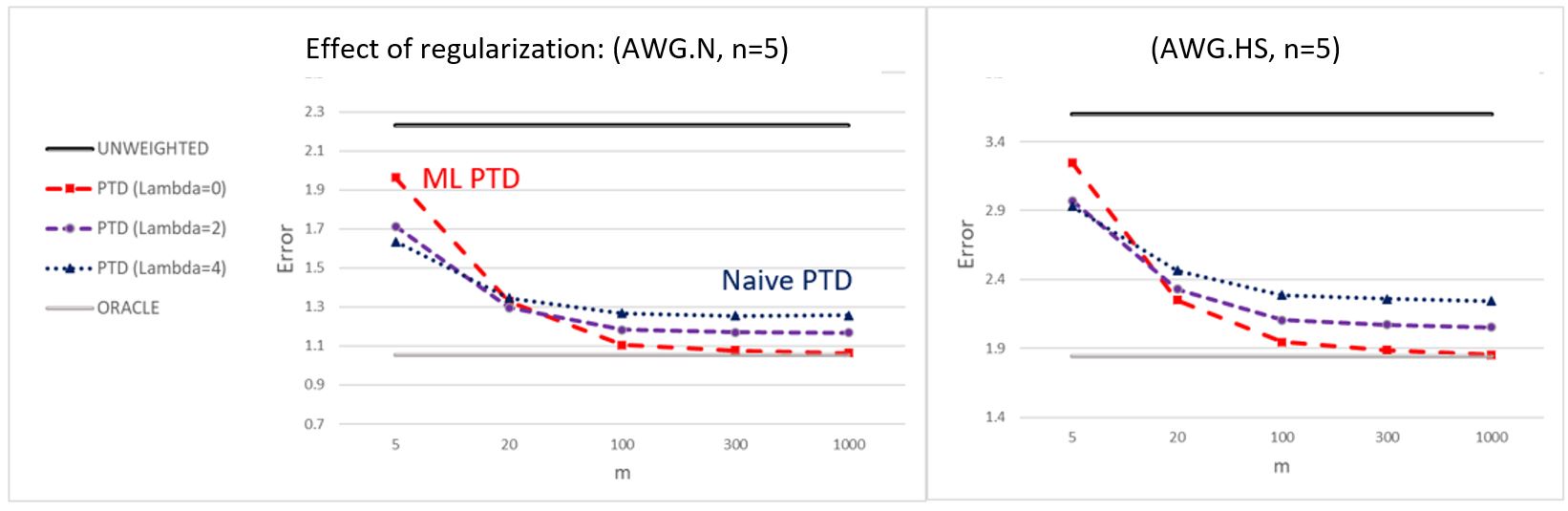}
    \caption{A comparison of the aggregation error of \PTD  with different regularization factors. Results are shown for synthetic data in the AWG model under two different distributions $\calT$.}
    \label{fig:PTD_types}
\end{figure}

\section{Sample Complexity and Consistency}\label{apx:sample}
In the main text we wrote that the error in estimating the expected disparity $E[\pi_i|t_i]$ is decreasing linearly with the sample size. Here we specify a formal upper bound. 
\begin{theorem}\label{thm:pii}
Suppose the fourth moment of each distribution  $\tilde \calY(\bm z,t), t\in \calT$ is bounded, and that questions are i.i.d. conditional on worker's type. That is, $\tilde \calY(\bm z,t)= \calZ(\bm z,t)^m$ for some some distribution $\calZ$.  Denote $\mu_{\bm s_i,\bm z}:=E_{S\sim \calY(\bm z,\calT)^n}[\pi_i(S)\mid \bm z,t_i]$ for some $\bm z\in\mathbb R^m,t_i\in \calT$. If $n,m>\Omega(\frac{1}{\delta\epsilon^2})$, we have: 

\labeq{pii}
{\Pr_{S\sim \calY(\bm z,\calT)^n}[|\pi_i(S)-\mu_{\bm s_i,\bm z}|>\epsilon \mid \bm z,t_i]<\delta.}

More specifically, let $M_t$ ($t=1,2,3,4$) denote the supremum of the $t$-th moment of any distribution  $\tilde \calY(\bm z,t), t\in \calT$. Then the bound in Eq.~\eqref{eq:pii} holds whenever 
 $n>\dfrac{24(4(M_2)^2+4M_2+2M_4)}{\delta\epsilon^2}+1$ and $m>\max\{\dfrac{12(M_4+4M_1M_3+ 4(M_1)^2 M_2)}{\delta\epsilon^2},\dfrac{3M_4}{\delta}\}$.
\end{theorem}
\begin{proof} W.l.o.g.~let $\bm z=\bm 0$. The proofs for other cases are  similar.
For any $\bm s_i$, let\\
$\mu_{\bm s_i,\bm z}=E_{S\sim \calY(\bm z,\calT)^n}[\pi_i(S)|\bm s_i,\bm z]=E_{S_{-i}\sim \calY(\bm z,\calT)^{n-1}}[\pi_i(S_{-i},\bm s_i)|\bm z]$. It follows from definition that $E(\pi_i(S))=E_{\bm s_i\sim \calY(\bm z,t_i)}(\mu_{\bm s_i,\bm z})$.
We define the following three events.
\begin{itemize}
    \item  \textbf{Event A.}  $|\pi_i(S)-\mu_{\bm s_i,\bm z}|> \epsilon_A$. The event will be used conditioned on $\bm s_i$ satisfying some properties.
    \item  \textbf{Event B.} $\left|\dfrac{\|\bm s_i\|_2^2}{m} - E_{\bm s_i\sim \calY(\bm z,t_i)}(\dfrac{\|\bm s_i\|_2^2}{m})\right| > \epsilon_B$. The event requires  $\dfrac{\|\bm s_i\|_2^2}{m}$ to be not too close to its mean, where $\bm s_i$ is drawn from $\calY(\bm z,t_i)$. If this event does not hold, then  $\dfrac{\|\bm s_i\|_2^2}{m}\le  M_2+\epsilon_B$. Later in the proof we will set $\epsilon_B=1$.
    \item \textbf{Event C.} $|\mu_{\bm s_i,\bm z} - E_{\bm s_i}(\mu_{\bm s_i,\bm z})|\le \epsilon_C$. The event requires $\mu_{\bm s_i,\bm z}$ to be not too close to its mean $E_{\bm s_i}(\mu_{\bm s_i,\bm z})=E (\pi_i(S))$, where the randomness comes from $\bm s_i$. $\mu_{\bm s_i,\bm z}$ is a function that only depends on $\bm s_i$.
\end{itemize}
In the following three claims, we give lower bounds on Events A (Claim~\ref{claim:eventa}), B (Claim~\ref{claim:eventb}), and C (Claim~\ref{claim:eventc}), respectively. The first claim states that the probability for Event A to hold is upper-bounded by a function that depends on $\frac{\|\bm s_i\|_2^2}{m}$, $n$, and $\epsilon_A$.
\begin{claim} 
\label{claim:eventa}
For any $\bm s_i$, we have $$\Pr(\text{Event A}|\bm s_i)=\Pr(|\pi_i(S)-\mu_{\bm s_i,\bm z}|> \epsilon_A|\bm s_i)<  \frac{8M_2\frac{\|\bm s_i\|_2^2}{m}+2M_4}{(n-1)\epsilon_A^2}$$
\end{claim}
\begin{proof}
For any $\bm s_i$, $\pi_i(S|\bm s_i)$ can be seen as average of $n-1$ i.i.d.~samples for $d(\bm s_i,\bm x)$, where $\bm x = (x_1,\ldots,x_m)$ is distributed as $\calY(\bm z,\calT)$. For any fixed $\bm s_i$, $d(\bm s_i,\bm x) = \frac{1}{m}\|\bm s_i\|_2^2 -\frac 2m \bm s_i\cdot\bm x +\frac{1}{m}\|\bm x\|_2^2$, where $\|\cdot\|_2$ is the $L_2$ norm.  
Therefore, the variance of $d(\bm s_i,\bm x)$ given $\bm s_i$ is the same as the variance of $-\frac 2m \bm s_i\cdot\bm x +\frac{1}{m}\|\bm x\|_2^2$, calculated as follows.
\begin{align*}
   VAR[d(\bm s_i,\bm x)|\bm s_i] &= \frac{1}{m^2}VAR[-2 \bm s_i\cdot\bm x +\|\bm x\|_2^2|\bm s_i]\\
& \le \frac1{m^2}E\{(-2 \bm s_i\cdot\bm x +\|\bm x\|_2^2)^2|\bm s_i\}\\
& \le \frac1{m^2}E\{2m(\sum_{j=1}^m(4s_{ij}^2x_j^2+x_j^4))|\bm s_i\} \text{(Cauchy-Schuwarz)}\\
&\le \frac2{m}(4M_2\sum_{j=1}^m\|\bm s_i\|_2^2+mM_4) = 8M_2\frac{\|\bm s_i\|_2^2}{m}+2M_4
\end{align*}
Recall that $M_2$ (respectively, $M_4$) is the maximum second (respectively, fourth) moment of random variables in $\calY(\bm z,\calT)$. 
The claim follows after Chebyshev's inequality.


\end{proof}

\begin{claim}\label{claim:eventb} 
For any $t_i$, we have 
$$\Pr_{\bm s_i\sim\calY(\bm z,t_i)}(\text{Event B}) = \Pr_{\bm s_i\sim\calY(\bm z,t_i)}(\left|\dfrac{\|\bm s_i\|_2^2}{m} - E_{\bm s_i\sim \calY(\bm z,t_i)}(\dfrac{\|\bm s_i\|_2^2}{m})\right| > \epsilon_B)< \frac{M_4}{m\epsilon_B^2}$$
\end{claim}
\begin{proof} Note that the components of $\bm s_{i}$ are i.i.d., whose variances are no more than $M_4$. Therefore, the variance of $\frac{\|\bm s_i\|_2^2}{m}$ is no more than $ \frac{M_4}m)$. The claim follows after Chebyshev's inequality.
\end{proof}

\begin{claim}\label{claim:eventc}
Let $C_3=M_4+4M_1M_3+ 4(M_1)^2 M_2$. 
For any $t_i$, we have 
$$\Pr_{\bm s_i\sim\calY(\bm z,t_i)}(\text{Event C}) = \Pr_{\bm s_i\sim\calY(\bm z,t_i)}(|\mu_{\bm s_i,\bm z} - E_{\bm s_i}(\mu_{\bm s_i,\bm z})|> \epsilon_C)< \frac{C_3}{m\epsilon_C^2}$$
\end{claim}
\begin{proof} For any given $\bm s_i$, we have
\begin{align*}
    \mu_{\bm s_i,\bm z} = & \frac 1m E_{\bm x\sim\calY(\bm z,\calT)}\{\sum_{j=1}^m(s_{ij}-x_j)^2\}\\
    =& \frac 1m E_{\bm x\sim\calY(\bm z,\calT)}\{\sum_{j=1}^ms_{ij}^2+\sum_{j=1}^mx_j^2+2\sum_{j=1}^ms_jx_j\}\\
    = & \frac{\|\bm s_i\|_2^2}{m} + \frac 1m\sum_{j=1}^mE(x_j^2)+\frac 2m\sum_{j=1}^ms_jE(x_j)\\
    \le &  \frac{\|\bm s_i\|_2^2}{m} +  M_2+\frac {2M_1}m\sum_{j=1}^ms_j 
\end{align*}
Recall that $M_1$ (respectively, $M_2$)  is the superium first (respectively, second) moment of random variables in $\calT$.
Now let us consider the case where $\bm s_i$ is generated from $\calY(\bm z,t_i)$. We have
\begin{align*}
    VAR(\mu_{\bm s_i,\bm z}) = & VAR(\frac{\|\bm s_i\|_2^2}{m} +  M_2+\frac {2M_1}m\sum_{j=1}^ms_j)\\
    =& \frac 1{m^2} VAR(\sum_{j=1}^m ( \bm s_{ij}^2  +   {2M_1}s_{ij}))\\
    =&\frac 1{m^2} \sum_{j=1}^m VAR( s_{ij}^2  +   {2M_1}s_{ij}) \ \ \ \ \text{(independence of $s_{ij}$'s)}\\
    \le & \frac 1m E_{s_{i1}}(s_{i1}^2  +   {2M_1}s_{i1})^2\le \frac 1m (M_4+4M_1M_3+ 4(M_1)^2 M_2)
\end{align*}
The claim follows after Chebyshev's inequality.
\end{proof}
Let $\epsilon_A  =\epsilon_C=\frac\epsilon 2$ and $\epsilon_B = 1$. We have:
\begin{align*}
    \Pr(|\pi_i-E(\pi_i)|>\epsilon) <&
\Pr(|\pi_i-\mu_{\bm s_i,\bm z}|>\frac{\epsilon}{2})+   \Pr(| \mu_{\bm s_i,\bm z}-E(\pi_i)|>\frac{\epsilon}{2})\\
=& \Pr(\text{Event A}) + \Pr( \text{Event C}) \\
\le&  \Pr( \text{Event A})+ \frac{4C_3}{m\epsilon^2} \ \ \ \ \text{(Claim~\ref{claim:eventc})}\\
\le& \Pr( \text{Event A} \& \neg \text{Event B}) +\Pr( \text{Event A}\&  \text{Event B}) + \frac{4C_3}{m\epsilon^2}\\
\le &\Pr( \text{Event A}|\neg \text{Event B})\Pr( \neg \text{Event B}) +\Pr( \text{Event B})+ \frac{4C_3}{m\epsilon^2}\\
\le &\Pr( \text{Event A}| \neg \text{Event B}) +\frac{M_4}{m\epsilon_B^2}+ \frac{4C_3}{m\epsilon^2} \ \ \ \ \text{(Claim~\ref{claim:eventb})}\\
\le & \frac{4(8M_2(M_2+1)+2M_4)}{(n-1)\epsilon^2}+\frac{M_4}{m}+ \frac{4C_3}{m\epsilon^2} \ \ \ \ \text{(Claim~\ref{claim:eventa} and $\epsilon_B=1$)}
\end{align*}
When $m$ and $n$ satisfy the conditions described in the theorem, each of the three parts is no more than $\frac\delta3$, which proves the theorem.
\end{proof}

	\paragraph{Consistency for \PTD with real-valued data} \label{apx:cont_consistent}

The theorem above, together with Cor.~\ref{cor:AK_sym} says that with enough samples, the fault under symmetric noise will be estimated correctly. 

The next theorem says that a good approximation of $\vec f$ entails a good approximation of $\bm z$, and thus guarantees that \PTD eventually gets closer to the truth with enough samples, even outside the AWG model.

Given a fault estimation $\hatv f=(\hat f_1,\ldots,\hat f_n)$, we denote the \emph{Aitkin weights} $w^A=w^A(\hatv f):=\frac{1}{\hatv f}$, following \cite{aitkin1935least}.
	
	Recall that $\hatvm z=\frac1n\sum_{i\in N}\frac{1}{\hat f_i} \bm s_i$ is the output of \PTD; and that  $\bm x^*=\frac1n\sum_{i\in N}\frac{1}{f_i} \bm s_i$ is the best linear unbiased  estimator of $\bm z$ by \cite{aitkin1935least}. 

	\begin{lemma}\label{lemma:conf_p_i}
Suppose that $\hat f_i \in (1\pm \delta) f_i$ for some $\delta\in[0,1]$. Then $w_i$ is  in the range $[ (1-\delta)w^A_i, i (1+2\delta)w^A_]$. 
\end{lemma}
\begin{proof} 
    \begin{align*}
    w_i &= \frac{1}{\hat f_i} \geq \frac{1}{f_i (1+\delta)} = \frac{1-\delta}{f_i (1+\delta)(1-\delta)}\\
    &= \frac{1-\delta}{f_i(1-\delta^2)} \geq \frac{1-\delta}{f_i} = (1-\delta) w^A_i.
    \end{align*}
and
    \begin{align*}
    w_i &= \frac{1}{\hat f_i}\leq \frac{1}{f_i (1-\delta)} = \frac{1+2\delta}{f_i (1-\delta)(1+2\delta)}\\
    &=  \frac{1+2\delta}{f_i (1+\delta-2\delta^2)} \leq \frac{1+2\delta}{f_i}=(1+2\delta)w^A_i. 
    \end{align*}
\end{proof}
	
	\begin{theorem}\label{thm:delta_cont}    Suppose that $\forall i\in N, \hat f_i \in (1\pm \delta) f_i$ for some $\delta\in[0,1]$; it holds that\\ 
    $d(\hatvm z,\bm z) \leq d(\bm x^*,\bm z) + O(\delta\cdot \max_{i,j}(s_{ij})^2)$.
	\end{theorem}
	
	\begin{proof}
    Note first that by Lemma~\ref{lemma:conf_p_i}, each $w_i$ is  in the range $[w^A_i (1-\delta), w^A_i (1+2\delta)]$.
    W.l.o.g., $\sum_{i\in N}w^A_i=1$. Thus
    $$ (1-\delta)\leq \sum_{i'}\hat w_{i'} \leq (1+2\delta),$$
    and 
    $$\frac{1}{\sum_{i'}\hat w_{i'}} \in (1-4\delta,1+4\delta)$$ 
    as well.
    
    Denote $\overline s = \max_{ij} s_{ij}$. For each $j \leq m $,
    \begin{align*}
        (\hat z_j-z_j)^2& = (\frac1{\sum_{i'}\hat w_{i'}}\sum_{i\in N}\hat w_i s_{ij} - z_{j})^2\\
        &=(\frac1{\sum_{i'}\hat w_{i'}}\sum_{i\in N}\hat w^A_i(1+\tau_i) s_{ij} - z_{j})^2\tag{for some $\tau_i\in [-2\delta,2\delta]$}\\
        &= (\sum_{i\in N} w^A_i(1+\tau'_i) s_{ij} - z_{j})^2\tag{for some $\tau'_i\in [-8\delta,8\delta]$}\\
        &=\sum_{i\in N}\sum_{i'\in N}w^A_i(1+\tau'_i)s_{ij}w^A_{i'}(1+\tau'_{i'})s_{i'j}+z_j^2 -2z_j\sum_{i\in N}\hat w^A_i(1+\tau'_i) s_{ij} \\
        &=\sum_{i\in N}\sum_{i'\in N}w^A_i s_{ij}w^A_{i'}s_{i'j}(1+\tau'_i + \tau'_{i'} +\tau'_i\tau'_{i'})+z_j^2 -2z_j\sum_{i\in N}\hat w^A_i(1+\tau'_i) s_{ij} \\
        &=\sum_{i\in N}\sum_{i'\in N}w^A_i s_{ij}w^A_{i'}s_{i'j}(1+\tau''_{ii'})+z_j^2 -2z_j\sum_{i\in N} w^A_i(1+\tau'_i) s_{ij}
    \end{align*}
    For $|\tau''_{ii'}| <25 \delta$.
    Similarly, 
    $$ (x^*-  z_j)^2 = \sum_{i\in N}\sum_{i'\in N}w^A_i s_{ij}w^A_{i'}s_{i'j}+z_j^2 -2z_j\sum_{i\in N} w^A_i s_{ij}.$$
    Next,
    \begin{align*}
     d&(\hatvm z,\bm z)-d(\bm x^*,\bm z) = \frac1m\sum_{j\leq m}\left((\hat z_j-z_j)^2-(x^*_j-z_j)^2\right)\\
     & =\frac1m\sum_{j\leq m}\left(\sum_{i\in N}\sum_{i'\in N}w^A_i s_{ij}w^A_{i'}s_{i'j}\right)\tau''_{ii'} + 2z_j\sum_{i\in N}w^A_is_{ij}\tau'_i\\
     &\leq 25\delta \frac1m\sum_{j\leq m}\left(\sum_{i\in N}\sum_{i'\in N}w^A_i s_{ij}w^A_{i'}s_{i'j}\right) + 16\delta \sum_{i\in N}w^A_i(\ol{\bm s})^2 \\
     &\leq 25\delta \frac1m\sum_{j\leq m}\sum_{i\in N}\sum_{i'\in N}w^A_i w^A_{i'}(\overline s)^2  + 16\delta \sum_{i\in N}w^A_i(\overline s)^2\\
     &\leq      50\delta (\overline s)^2,
    \end{align*}
    as required.
\end{proof}


 
\rmr{A cleaner proof would be using Thm~\ref{thm:pii}. However we can not use the union bound since this would require $n>\Omega(n)$, and we cannot use independence since the errors of $\pi_i,\pi_{i'}$ are correlated: if one worker is an outlier, then this may distort all of the disparities. Is it possible to remove the dependency of $n$ from Thm.~\ref{thm:pii}? or at least make it sublinear?}  

\input{algorithms.tex}

\section{Empirical Appendix}\label{apx:empirical_more}
\subsection{Datasets}
\begin{description}
\item[Buildings] Subjects were asked to mark the height of the building in the picture on a slide bar (the ranking dataset was obtained from the same data). Participants were given short instructions,\rmr{ removed to maintain anonymity \footnote{see survey in \href{https://goo.gl/W47X5R}{https://goo.gl/W47X5R}.} }  then they had to answer $m=25$ questions. We recruited participants through Amazon Mechanical Turk. 
Participants who answered correctly the gold standard question received a payment of $\$0.3$.  Participants did not receive bonuses for accuracy.  The study protocol was approved by the Institutional Review Board at our institution.

\item[Triangles] These two datasets are from a study of people's geometric reasoning~\cite{hart2018statistical}.\footnote{\url{https://github.com/StatShapeGeometricReasoning/StatisticalShapeGeometricReasoning}}
Participants in the study were shown the base of a triangle (two vertices and their angles) and were asked to position the third vertex. That is, their answer to each question is an x-coordinate and y-coordinate of a vertex. In our analysis we treat each of the coordinates as a separate question. 

\item[GG, DOGS, FLAGS] 
Three datasets from \cite{shah2015double}, in those experiments workers had to identify an object in pictures.  Workers got paid by the Double or Nothing payment scheme described in their paper (they had a fourth dataset, \textbf{HeadsOfCountries}, where almost all subjects got perfect answers, and hence we did not use this dataset).
\item[DOTS, PUZZ] Datasets from \cite{mao2013better}, in those experiments workers had to rank four alternatives.  The two datasets consists 6400 rankings from 1300 unique workers.  Workers got paid 0.1\$ per ranking.
\item[Predict] Curated by \cite{mandal2020}. Subjects were asked to assess the likelihood that 18 events will occur (in particular to predict whether or not they will occur), e.g., ``Will D.C. United win the D.C. United vs. FC Dallas soccer game  on Sat Oct. 13th?"
Subject were assigned to four treatments of the payment method: (1) according to Brier scoring rule (207 workers); (2) According to peer-prediction (220 workers); (3)+(4) according to combinations of the two (238 and 226 workers, respectively). 
 The datasets were collected at different dates (8-13 Oct.). 
 We display results for the first date in the dataset. The results for other dates are similar.
 \item[English, Chinese, Science, IT, Medicine, Pokemon] Datasets from \cite{kawase2019graph}.
 \item[EMO] The Emotions dataset from \cite{snow2008cheap}.
 
 \item[Translations]  The dataset, that was made available to us by the authors of \cite{braylan2020modeling}, contains translations of 100 sentences from Japanese to English (in \cite{braylan2020modeling} they used a similar dataset of Urdu-English). Each sentence has 10 translations by different workers, and was used (in our study) as a separate dataset. The distance to the ground truth (and among pairs of workers) was determined by GLEU score~\cite{wu2016google}.

\item[ETCH-a-CELL] The  ETCH-a-CELL dataset contains bitmaps of the outline of a tumor in 2D slices of a cell~\cite{spiers2021deep}. There are 235 slices, with between 10 and 30 labels of each slice.
To measure the distance between pairs of labels, we first filled the hollow outline using standard algorithms, and then  used Jaccard distance on the filled shapes. We aggregated labels either by taking the best worker (according to the algorithm), or using weighted pixelwise-majority, where weights are determined by the used truth-discovery algorithm. 

\end{description}
All datasets curated by us are attached to the submission.


\newpar{Synthetic data}
For synthetic  data we generated instances from the AWG model (real-valued); the one-coin model (categorical); or Mallows model  with  6 candidates (ranking). 
\begin{description}
\item[Real-valued]The ground truth in the AWG model was sampled from $N(0,2)$.  For the proto-population $\calT$  we used either a truncated Normal distribution with $f_i\sim N(1,0.5)$, clipped to nonnegative values (SYN.N); or noisy variant of the ``Hammer-Spammer" distribution (SYN.HS). The fault values for hammers and spammers were 0.2 and 1, respectively.
\item[Categorical, k=2] Proto-population was a truncated Normal distribution $N(0.35,0.15)$; or a Hammer-spammer distribution (20\% hammers, fault levels 0.2 and 0.5).
\item[Categorical, k=4] Proto-population was a truncated Normal distribution $N(0.6,0.3)$;
\item[Ranking] We used Mallows distribution with parameter $\phi$.  The proto population (distribution of $\phi_i$) was either a clipped Normal distribution (we used $N(0.65,0.15)$ or a Hammer-spammer distribution (20\% hammers with $\phi_i=0.3$, and 80\% spammers with $\phi_i=0.99$ which is just slightly better than random).
\end{description}

\medskip


\paragraph{Correlation of disparity and error}
\begin{figure}
    \centering
    \includegraphics[width=0.8\textwidth]{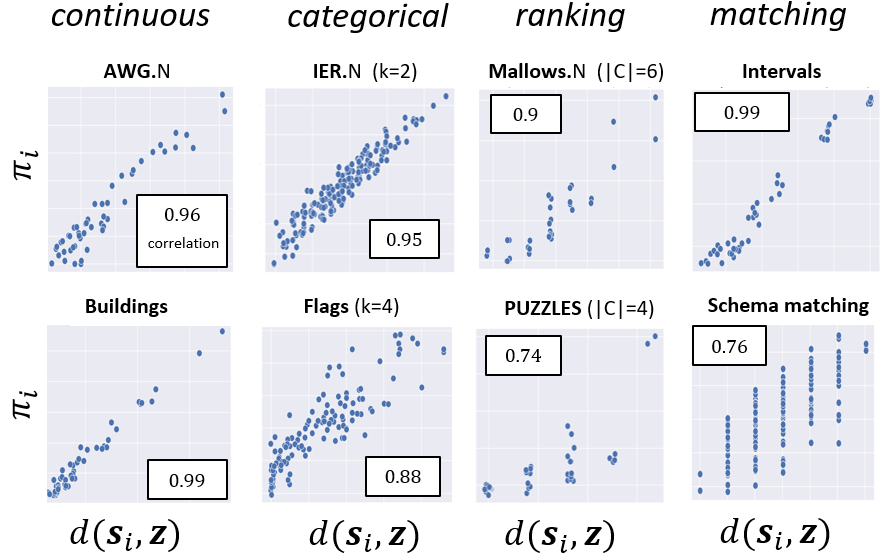}
    \caption{Each scatterplot presents all workers in a single instance, with their disparity $\pi_i$ vs. their real error $d(\bm s_i,\bm z)$.}
    \label{fig:correlation_details}
\end{figure}
Fig.~\ref{fig:corr_all_main} in the main text shows the correlation of workers' empirical disparity $\pi_i$ with their distance from the ground truth. 

Fig.~\ref{fig:correlation_details} is the same figure, but now specifying the dataset we used, and also synthetic distributions. 

The two datasets not already described above are \textbf{schemaMatching}, where users were asked to align the attributes of two database schemata~\cite{shraga2018type}; and \textbf{Intervals} which is a synthetic dataset we constructed where $\bm z$ is a set of intervals, and each report $\bm s_i$ is a noisy version of it. Both use the Jaccard distance metric. 





\subsection{Additional Results}
Here we add tables and figures showing performance of all algorithms on the same datasets under different numbers of workers and questions.



\begin{table}[t]
    \centering
    \begin{tabular}{|l|l||l|l|l|l||l|l|l||l|l|l|}
    \hline
    \multicolumn{12}{|c|}{Real-valued data, $m=15, n=30$}\\
    \hline
Real-valued&&{\CRH}&{\CATD}&{\GTM}&{\KDEm}&{\TTEXP}&{\TT}&{\MAS}&{\PTDD}&{\PTDAK}&{\IPTD}\\					\hline
SYN.N&\irrel&{-33\%}&{-61\%}&{-27\%}&\aboveUA{+7\%}&{-12\%}&{-22\%}&\aboveUA{+52\%}&{-38\%}&\best{-70\%}&{-43\%}\\					
BUILD&\irrel&{-19\%}&{-20\%}&{-22\%}&\best{-34\%}&{-29\%}&\aboveUA{+68\%}&{-25\%}&{-21\%}&{-25\%}&\inrange{-33\%}\\					
TRI1&\irrel&{-18\%}&{-7\%}&{-18\%}&\aboveUA{+8\%}&\best{-27\%}&\aboveUA{+68\%}&\aboveUA{+27\%}&{-18\%}&{-8\%}&{-21\%}\\					
TRI2&\irrel&{-11\%}&{-22\%}&{-8\%}&{-19\%}&{-17\%}&\best{-26\%}&\aboveUA{+2\%}&{-11\%}&{-8\%}&{-17\%}\\	
\hline
    \multicolumn{12}{|c|}{Real-valued data, $m=50, n=10$}\\
\hline
    Real-valued&&{\CRH}&{\CATD}&{\GTM}&{\KDEm}&{\TTEXP}&{\TT}&{\MAS}&{\PTDD}&{\PTDAK}&{\IPTD}\\					\hline
SYN.N&\irrel&{-37\%}&{-32\%}&{-38\%}&\aboveUA{+32\%}&{-15\%}&{-14\%}&\aboveUA{+77\%}&{-34\%}&\best{-62\%}&{-40\%}\\					
TRI1&\irrel&\inrange{-31\%}&{-5\%}&\best{-32\%}&\aboveUA{+1\%}&{-28\%}&\aboveUA{+26\%}&\aboveUA{+78\%}&{-29\%}&{-10\%}&{-30\%}\\					
TRI2&\irrel&{-14\%}&{-10\%}&{-12\%}&{-0\%}&\best{-19\%}&{-11\%}&\aboveUA{+21\%}&{-11\%}&{-10\%}&\inrange{-17\%}\\					
EMO&\irrel&\aboveUA{+4\%}&\aboveUA{+47\%}&\aboveUA{+6\%}&\aboveUA{+76\%}&\aboveUA{+20\%}&\aboveUA{+66\%}&\aboveUA{+15\%}&\aboveUA{+4\%}&\aboveUA{+12\%}&\aboveUA{+7\%}\\	
\hline
    \end{tabular}
    \caption{$m=50, n=10$}
    \label{tab:my_label}
\end{table}

\begin{table}[]
    \centering
        \begin{tabular}{|l|l||l|l|l|l||l|l|}
    \hline
Ranking&v. rule&{\DTD}&{\TTEXP}&{\TT}&{\MAS}&{\PTDD}&{\IPTD}\\		\hline	
SYN.HS&Veto&\aboveUA{+2\%}&{-5\%}&{-0\%}&\best{-9\%}&{-6\%}&\inrange{-7\%}\\			
SYN.N&Veto&\aboveUA{+5\%}&{-0\%}&\aboveUA{+51\%}&{-3\%}&\inrange{-9\%}&\best{-10\%}\\			
BUILD&Veto&\aboveUA{+1\%}&\aboveUA{+6\%}&\aboveUA{+51\%}&{-5\%}&\best{-7\%}&\inrange{-6\%}\\			
DOTS3&Veto&\aboveUA{+3\%}&{-5\%}&\aboveUA{+10\%}&{-8\%}&{-10\%}&\best{-14\%}\\			
DOTS5&Veto&\aboveUA{+5\%}&{-8\%}&\aboveUA{+16\%}&{-7\%}&{-14\%}&\best{-17\%}\\			
DOTS7&Veto&\aboveUA{+4\%}&{-9\%}&\aboveUA{+52\%}&{-5\%}&{-16\%}&\best{-21\%}\\			
DOTS9&Veto&\aboveUA{+10\%}&{-7\%}&\aboveUA{+48\%}&{-7\%}&{-18\%}&\best{-23\%}\\			
PUZZ5&Veto&\aboveUA{+4\%}&{-14\%}&\aboveUA{+40\%}&{-16\%}&{-22\%}&\best{-29\%}\\			
PUZZ7&Veto&\aboveUA{+7\%}&{-4\%}&\aboveUA{+59\%}&{-11\%}&\inrange{-21\%}&\best{-22\%}\\			
PUZZ9&Veto&\aboveUA{+3\%}&{-7\%}&\aboveUA{+28\%}&{-15\%}&{-15\%}&\best{-19\%}\\			
PUZZ11&Veto&\aboveUA{+6\%}&{-3\%}&\aboveUA{+16\%}&{-6\%}&{-10\%}&\best{-13\%}\\			
	\hline		
SYN.N&Best worker&{-12\%}&{-21\%}&{-72\%}&{-87\%}&\inrange{-89\%}&\best{-89\%}\\			
BUILD&Best worker&{-9\%}&{-20\%}&{-52\%}&{-51\%}&\inrange{-57\%}&\best{-57\%}\\			
DOTS3&Best worker&{-19\%}&{-9\%}&{-43\%}&{-53\%}&\inrange{-59\%}&\best{-59\%}\\			
PUZZ5&Best worker&{-25\%}&{-37\%}&{-92\%}&\inrange{-94\%}&\inrange{-93\%}&\best{-94\%}\\			
	\hline		
SYN.N&Kemeny-Young&{-2\%}&\aboveUA{+4\%}&\aboveUA{+176\%}&\aboveUA{+1\%}&\best{-8\%}&\inrange{-8\%}\\			
BUILD&Kemeny-Young&{-1\%}&{-1\%}&\aboveUA{+12\%}&\best{-4\%}&{-1\%}&{-1\%}\\			
DOTS3&Kemeny-Young&\aboveUA{+1\%}&\aboveUA{+2\%}&\aboveUA{+37\%}&\inrange{-2\%}&\best{-4\%}&\inrange{-2\%}\\			
PUZZ5&Kemeny-Young&{-2\%}&\aboveUA{+1\%}&\aboveUA{+6\%}&\aboveUA{+23\%}&{-23\%}&\best{-27\%}\\			

\hline
\end{tabular}
    \caption{Results (RI) on ranking data, $n=10$, additional voting rules.}
    \label{tab:RANK_n10}
\end{table}

\begin{table}[]
    \centering
        \begin{tabular}{|l|l||l|l|l|l||l|l|}
    \hline
Ranking&v. rule&{\DTD}&{\TTEXP}&{\TT}&{\MAS}&{\PTDD}&{\IPTD}\\		\hline	
SYN.HS&Borda&{-20\%}&{-1\%}&\best{-58\%}&{-43\%}&{-7\%}&{-20\%}\\			
SYN.N&Borda&{-10\%}&{-2\%}&\aboveUA{+91\%}&\best{-27\%}&{-12\%}&{-17\%}\\			
BUILD&Borda&\aboveUA{+2\%}&\aboveUA{+0\%}&\aboveUA{+0\%}&\aboveUA{+0\%}&\aboveUA{+0\%}&\aboveUA{+0\%}\\			
DOTS3&Borda&\inrange{-0\%}&\aboveUA{+2\%}&\aboveUA{+44\%}&\aboveUA{+2\%}&\best{-2\%}&\aboveUA{+1\%}\\			
DOTS5&Borda&{-2\%}&{-2\%}&\aboveUA{+60\%}&\aboveUA{+25\%}&{-8\%}&\best{-10\%}\\			
DOTS7&Borda&\aboveUA{+0\%}&\aboveUA{+6\%}&\aboveUA{+127\%}&\aboveUA{+18\%}&\best{-12\%}&\inrange{-11\%}\\			
DOTS9&Borda&{-5\%}&\aboveUA{+3\%}&\aboveUA{+86\%}&{-22\%}&{-23\%}&\best{-30\%}\\			
PUZZ5&Borda&\aboveUA{+1\%}&{-7\%}&{-36\%}&{-40\%}&{-29\%}&\best{-44\%}\\			
PUZZ7&Borda&\aboveUA{+5\%}&\aboveUA{+18\%}&\aboveUA{+139\%}&\aboveUA{+18\%}&\best{-13\%}&{-2\%}\\			
PUZZ9&Borda&\aboveUA{+1\%}&{-6\%}&\aboveUA{+28\%}&{-9\%}&{-13\%}&\best{-20\%}\\			
PUZZ11&Borda&{-0\%}&\aboveUA{+3\%}&\aboveUA{+23\%}&\best{-20\%}&{-4\%}&{-5\%}\\			
	\hline		
SYN.N&Plurality&\aboveUA{+3\%}&{-8\%}&\aboveUA{+171\%}&{-11\%}&\inrange{-12\%}&\best{-14\%}\\			
BUILD&Plurality&{-2\%}&{-0\%}&\aboveUA{+61\%}&\best{-13\%}&{-6\%}&{-6\%}\\			
DOTS3&Plurality&\aboveUA{+6\%}&\aboveUA{+1\%}&\aboveUA{+91\%}&\inrange{-5\%}&\inrange{-5\%}&\best{-6\%}\\			
PUZZ5&Plurality&\aboveUA{+6\%}&{-16\%}&\aboveUA{+126\%}&{-23\%}&{-27\%}&\best{-37\%}\\			
	\hline		
SYN.N&Copeland&{-6\%}&{-9\%}&\aboveUA{+107\%}&{-18\%}&\best{-31\%}&\inrange{-31\%}\\			
BUILD&Copeland&\aboveUA{+2\%}&\inrange{-0\%}&\inrange{-0\%}&\aboveUA{+1\%}&\best{-1\%}&\inrange{-0\%}\\			
DOTS3&Copeland&{-1\%}&{-6\%}&\aboveUA{+10\%}&{-18\%}&\best{-23\%}&\inrange{-21\%}\\			
PUZZ5&Copeland&{-1\%}&{-20\%}&{-22\%}&{-41\%}&{-44\%}&\best{-46\%}\\			
\hline
\end{tabular}
    \caption{Results (RI) on ranking data, $n=30$.}
    \label{tab:RANK_n30}
\end{table}

\begin{figure}
    \centering
    \includegraphics[width=0.48\linewidth]{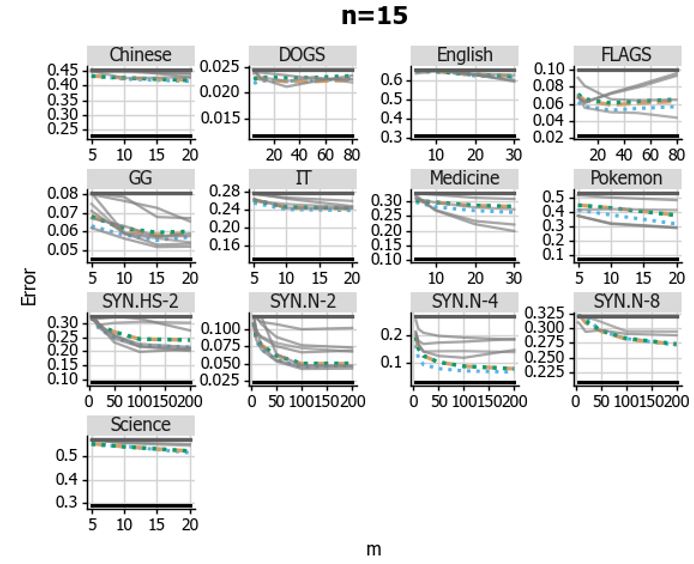}
    \includegraphics[width=0.48\linewidth]{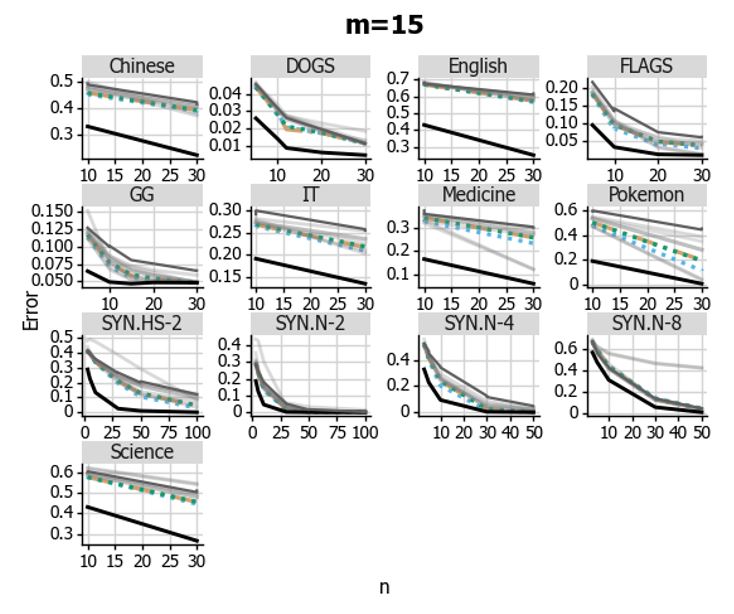}
    \caption{Error on all categorical datasets.}
    \label{fig:CAT_all}
\end{figure}

\begin{figure}
    \centering
    \includegraphics[width=0.48\linewidth]{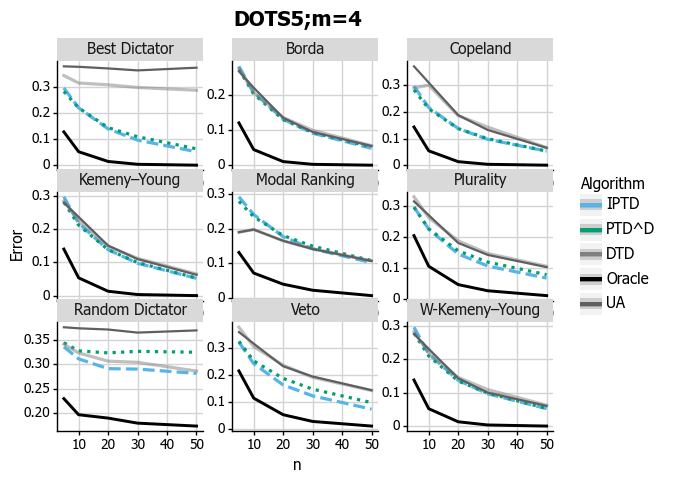}
    \includegraphics[width=0.48\linewidth]{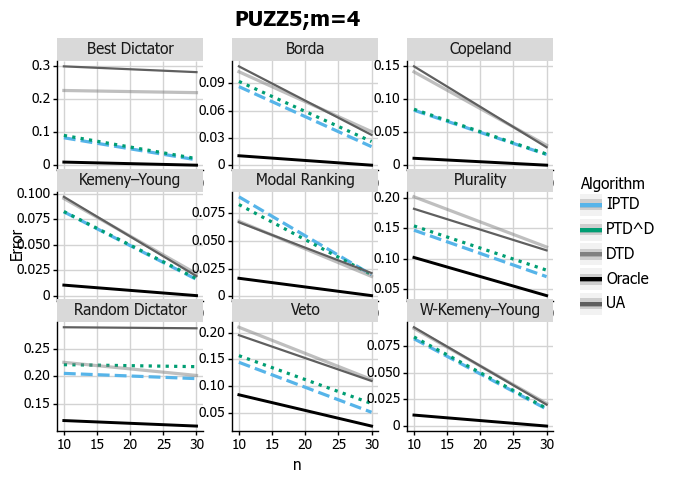}
    \includegraphics[width=0.48\linewidth]{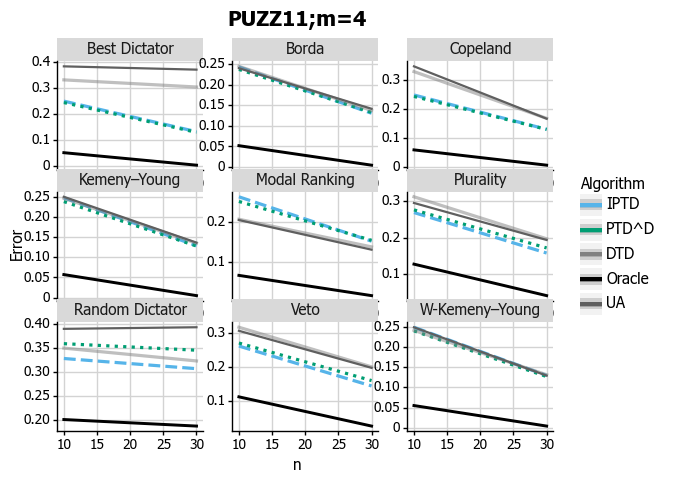}
    \includegraphics[width=0.48\linewidth]{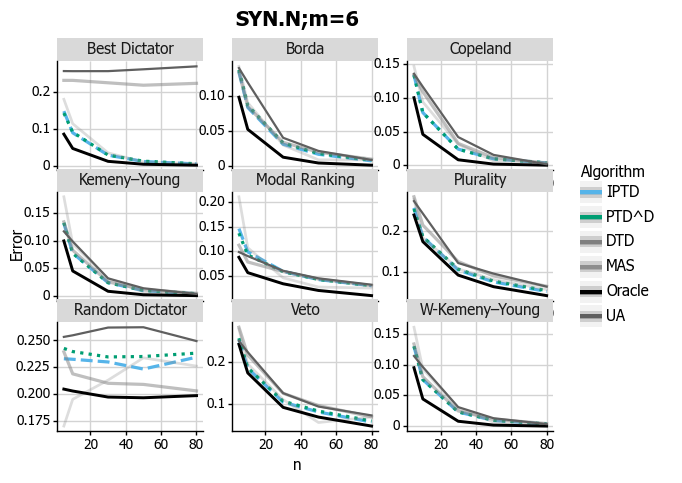}
    \caption{Error on  more ranking datasets, under all nine voting rules, 4 candidates.}
    \label{fig:RANK_more}
\end{figure}

\begin{figure}
    \centering
    \includegraphics{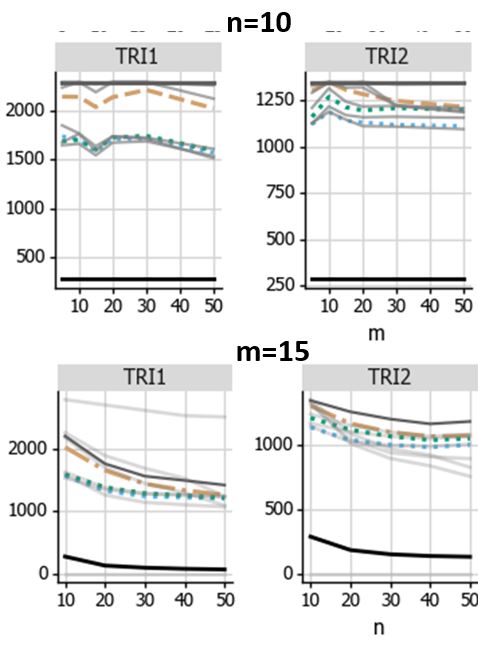}
    \caption{Error on the real-valued TRI datasets.}
    \label{fig:TRI_all}
\end{figure}

%% file: algorithms.tex
\section{Algorithms}\label{apx:algs}
\subsection{\PTD algorithms}

The following table specifies the three steps of our \PTD algorithm. Recall that the `default' \PTDD algorithm is independent of the domain.  

\begin{tabular}{||l||l||l|l||}
\hline
Step~1 \vph\vpl &\multicolumn{3}{|c||}{For all $i\in N$, $\pi_i\leftarrow \frac{1}{n-1}\sum_{i'\neq i}d(\bm s_i,\bm s_{i'})$}  \\
\hline
Implementation: \vph & \PTDD (all domains) &\PTDAK (real-valued)& \PTDAK (categorical)\\

\hline

Step~2\vpl \vph&& \multicolumn{2}{|c||}{$\bar\mu \leftarrow \frac{1}{2n}\sum_{i\in N}\pi_i$}\\
\cline{3-4}
 \vph \vpl& $\forall i\in N, \hat f_i \leftarrow \pi_i$ & $\forall i\in N,\hat f_i \leftarrow \pi_i-\frac{n}{n-1}\bar\mu$ & $\forall i\in N,\hat f_i \leftarrow \frac{\pi_i-\bar\mu}{1-\frac{k}{k-1}\bar\mu}$ \\ 
\hline
Step~3 \vph \vpl & $\forall i\in N, w_i\leftarrow \max_{i'}(\hat f_{i'})-\hat f_i$ &$\forall i\in N, w_i\leftarrow \frac{1}{\hat f_i}$& $\forall i\in N, w_i \leftarrow \log(\frac{(1-\hat f_i)(k-1)}{\hat f_i})$\\
\cline{2-4}
\vph&\multicolumn{3}{|c||}{Aggregate answers using weights $\vec w$}\\
\hline
\end{tabular}

In Step~2, we rely on the appropriate Anna-Karenina theorem to set the value of $\hat f_i$ in \PTDAK. The real-valued implementation relies on Theorem~\ref{thm:pAWG_opt} with no regularization ($\lambda=0$), and the categorical implementation on Cor~\ref{cor:AK_cat}.

In Step~3 we rely on known results from statistics regarding the best estimator of the ground truth when $\vec f$ is known: On \cite{aitkin1935least} for the AWG model (real-valued domain); and on \cite{grofman1983thirteen,ben2001optimal} for the one-coin model (categorical domain).
 
The aggregation function itself depends on the domain: we used mean (for real-valued); plurality (for categorical); and nine different voting rules (for ranking). 
On the Translations dataset where there is no aggregation, we simply used $\hatvm z=\bm s_{i^*}$ where $i^*$ is the worker with lowest $\pi_i$. On the Etch-a-Cell dataset we used pixel-wise majority of the bitmaps.

\begin{algorithm}[H]
\caption{\textsc{Iterative-Proxy-based-Truth-Discovery (\IPTD)}}
\label{alg:IPTD}
\KwIn{number of iterations $T$; dataset $S$}
\KwOut{Fault levels $\hatv f=(\hat f_i)_{i\in N}$; answers $\hatvm z$}
Compute $d_{ii'} \leftarrow d(\bm s_i,\bm s_{i'})$ for every pair of workers\;
Initialize $\vec \omega^0 \leftarrow \vec{1}$\;
\For{$t=1,2,\ldots,T$}{
For every worker $i\in N$,  set $\pi^t_i \leftarrow {\sum_{i'\neq i}\omega^{t-1}_{i'} d_{ii'}}$\;
For every worker $i\in N$, set $\omega^t_i \leftarrow \max_{i'}(\pi^t_{i'}) - \pi^t_i$\;
}

\For{every worker $i\in N$}{
	  Set $\hat f_i \leftarrow \pi^T_i$\;
Set $w_i\leftarrow \max_{i'}(\hat f_{i'})-\hat f_i$\;
}
Aggregate answers using weights $\vec w$\;
\end{algorithm}

\subsection{Benchmark algorithms}
\paragraph{Distance-based-truth-discovery (\DTD)} This is a simple baseline  that first uses the unweighted aggregation rule to compute an interim estimation $\bm z^0$ of the ground truth, then estimates the fault as $\hat f_i := d(\bm s_i, \bm z^0)$, and sets the weights as in Step~3 of \PTD. Various versions of \DTD appear in the literature, going back at least to \cite{grofman1983thirteen}. 

\medskip
For two algorithms, we used the original code by the authors:
\begin{description}
\item[Multidimensional Annotation Scaling (\MAS)]
We obtained the code of the MAS algorithm with the courtesy of the authors of \cite{braylan2020modeling}. We ran it without changing the meta-parameters. We did not partition datasets into subtasks. All meta parameters were set as recommended in the paper.
\item[Kernel Density Estimation from multiple sources (\KDEm)]For \KDEm we used the code provided with paper~\cite{wan2016truth} (\KDEm-d variation) for real-valued data.
\end{description}

We implemented the other algorithms based on the pseudocode given in the respective papers:
\begin{description}
\item[Conflict Resolution on
Heterogeneous Data (\CRH)] The implementation is following~\cite{li2014resolving}, using the solutions given in the paper for real-valued and categorical data;
\item[Belief Propagation Truth discovery (\BPTD)] This is an EM-style algorithm that iteratively estimates the probability of each binary label to be 1 or 0, and the competence of each worker~\cite{karger2011iterative}. The algorithm in the paper also had a random initialization step, that made all the results strictly worse. We thus used uniform initialization of weights as in \IPTD; 
\item[\GTM] Based on \cite{zhao2012probabilistic}. the paper takes a Bayesian approach and assumes an inverse Gamma prior distribution with parameters $(\alpha, \beta)$ for the sources' accuracy and a normal prior with parameters $(\mu_0, \sigma_0$) for the ground truth. In this paper, we used the parameters which have shown the best results among the different prior parameters the authors used for their empirical evaluation; 

\item[Dawid-Skene Truth Discovery (\DSTD)]Our implementation is based on the pseudocode in~\cite{gao2013minimax}. It does not include the initialization phase (we initialize uniform weights as in \IPTD);
\item[Eigenvector Truth Discovery (\EVTD)]~\cite{parisi2014ranking} (the symmetric version) This algorithm uses the off-diagonal covariance matrix of workers by first filling in values for the diagonal (there are several suggestions for this in the paper, we used the one recommended by the authors), then calculates the leading eigenvector of the matrix and uses it as weights; 
\item[Expert-Core (EXP)] The paper~\cite{kawase2019graph} divides the estimation problem into two parts: constructing a similarity graph (relevant only for Categorical labels); and finding a dense subgraph (relevant to any domain). In all non-categorical domains we used $\max d-d_{ii'}$ as the similarity.   
\item[Top-2] this is another variation of \TTEXP~\cite{kawase2019graph} (a fixed size of 2 expert core);
\item[\CATD]Based on \cite{li2014confidence}. The authors calculate a $1-\alpha$ confidence interval to their estimation of each of the sources' accuracy, and use the upper limit of each confidence interval to weigh a source's answer, in this paper we used $\alpha = .05$.
\end{description}

In all iterative algorithms (including \IPTD), we placed an upper bound of 15 on the number of iterations (typically algorithms converged much before the limit). In addition, we did not allow negative weights and replaced them with 0 \tsv{Should we mention that we didn't do any of the outlier removal for any algorithm?}.